\documentclass[letterpaper]{article} 
\usepackage{aaai25}  
\usepackage{times}  
\usepackage{helvet}  
\usepackage{courier}  
\usepackage[hyphens]{url}  
\usepackage{graphicx} 
\usepackage{natbib}  
\usepackage{caption} 
\usepackage{algorithm}
\usepackage{algorithmic}

\usepackage{bm} 
\usepackage{amsmath, amssymb, amsthm, epsfig, url, array}
\usepackage{cleveref}

\theoremstyle{plain}
\newtheorem{theorem}{Theorem}[section]
\newtheorem{property}[theorem]{Property}

\newtheorem{lemma}[theorem]{Lemma}

\theoremstyle{definition}
\newtheorem{definition}[theorem]{Definition}
\newtheorem{assumption}[theorem]{Assumption}

\crefname{theorem}{theorem}{theorems}
\crefname{lemma}{lemma}{lemmas}
\crefname{definition}{definition}{definitions}
\crefname{assumption}{assumption}{assumptions}
\crefname{corollary}{cororally}{corollaries}
\crefname{property}{property}{properties} 
\usepackage {threeparttable}
\usepackage{makecell, caption, booktabs, multirow, siunitx}
\def\diag{\mathop{\rm diag }\nolimits}

\usepackage{newfloat}
\usepackage{listings}
\DeclareCaptionStyle{ruled}{labelfont=normalfont,labelsep=colon,strut=off} 
\floatstyle{ruled}
\newfloat{listing}{tb}{lst}{}
\floatname{listing}{Listing}
\title{Learning from Summarized Data: \\ Gaussian Process Regression with Sample Quasi-Likelihood}
\author{
    Yuta Shikuri
}
\affiliations{
    Tokio Marine Holdings, Inc. \\ 
    Tokyo, Japan  \\ 
    shikuriyuta@gmail.com
}

\begin{document}

\maketitle

\begin{abstract}    
Gaussian process regression is a powerful Bayesian nonlinear regression method. 
Recent research has enabled the capture of many types of observations using non-Gaussian likelihoods.  
To deal with various tasks in spatial modeling, we benefit from this development. 
Difficulties still arise when we can only access summarized data consisting of representative features, summary statistics, and data point counts. 
Such situations frequently occur primarily due to concerns about confidentiality and management costs associated with spatial data. 
This study tackles learning and inference using only summarized data within the framework of Gaussian process regression. 
To address this challenge, we analyze the approximation errors in the marginal likelihood and posterior distribution that arise from utilizing representative features. 
We also introduce the concept of sample quasi-likelihood, which facilitates learning and inference using only summarized data. 
Non-Gaussian likelihoods satisfying certain assumptions can be captured by specifying a variance function that characterizes a sample quasi-likelihood function. 
Theoretical and experimental results demonstrate that the approximation performance is influenced by the granularity of summarized data relative to the length scale of covariance functions. 
Experiments on a real-world dataset highlight the practicality of our method for spatial modeling. 
\end{abstract}

\begin{figure*}[t]
    \begin{center}
      \includegraphics[width=160mm]{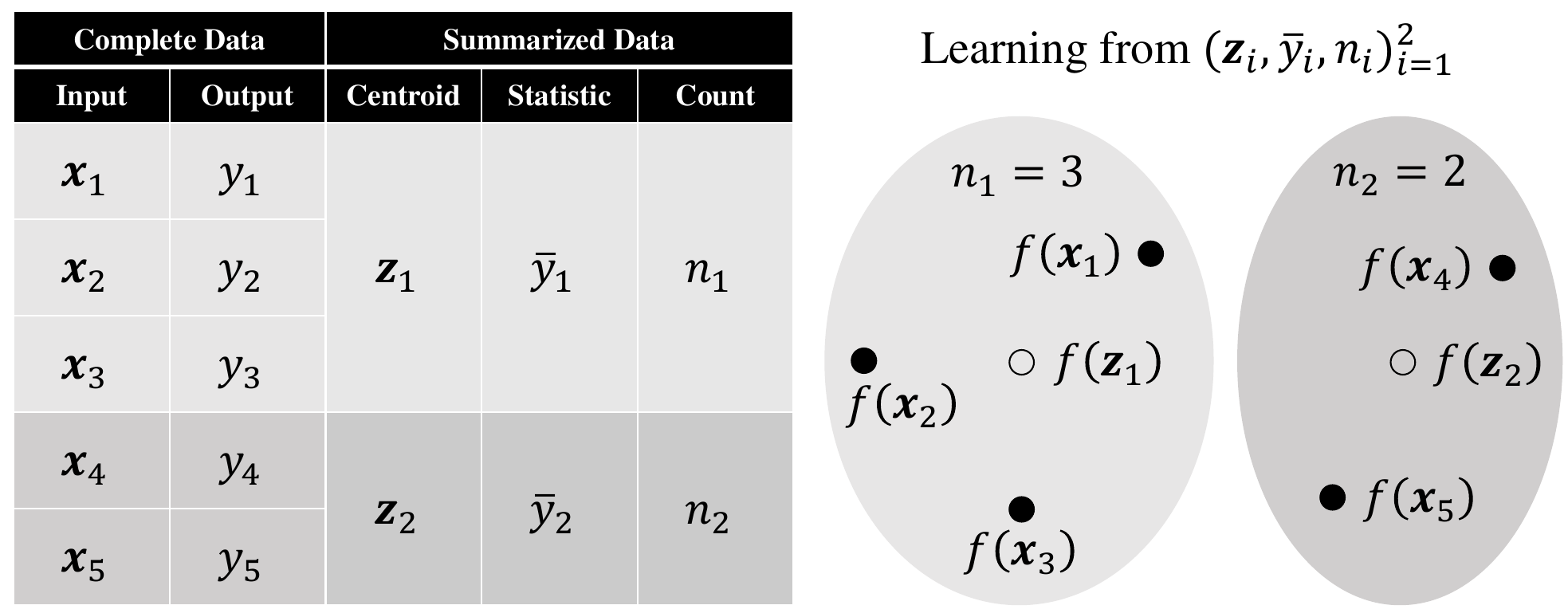}
    \end{center}
\caption{
    Overview of our approach. 
    We address Gaussian process regression using only summarized data, whihc includes representative features, summary statistics, and data point counts. 
    Our approach involves the input approximation and the introduction of sample quasi-likelihood under specific assumptions. 
    Specifying a variance function that characterizes a sample quasi-likelihood allows us to capture non-Gaussian likelihoods. 
}
\label{overview}
\end{figure*}

\section{Introduction}
Gaussian process regression is a Bayesian nonlinear regression method that can handle many types of observation data \citep{Carl}. 
A key application of Gaussian process regression is spatial modeling \citep{Noel} across various fields, such as geology, agriculture, marketing, and public health.
For example, it is used to estimate the distribution of soil moisture and regional medical expenses. 
Specifying likelihood functions according to these tasks impacts prediction performance. 
While a conventional choice is a Gaussian likelihood that is conjugate with a Gaussian process prior, recent research has focused extensively on applying Gaussian process regression to non-Gaussian likelihoods. 
The challenge of specifying non-Gaussian likelihoods lies in posterior distributions that lack closed-form expressions. 
Gaussian approximations of the posterior distribution are commonly used to overcome this challenge. 
Laplace approximation represents the posterior distribution as a Gaussian centered at the maximum a posteriori estimate \citep{Christopher}. 
Expectation propagation finds the parameters in a Gaussian that approximate the posterior distribution through an iterative scheme of matching moments \citep{Minka}. 
Variational inference maximizes a lower bound on the marginal likelihood via a Gaussian approximation \citep{Roni}.   
These methods for specifying non-Gaussian likelihoods allow us to handle various types of observations. 

Difficulties have been encountered when only summarized data is available. 
This situation arises particularly in the context of spatial modeling. 
Due to privacy concerns associated with location-specific observation information, spatial data is often summarized to include representative locations, summary statistics, and counts. 
Too fine locations might identify individuals even if direct identifiers are removed. 
Hence, the granularity of summarized data tends to be intentionally coarse to ensure that personal information (e.g., financial assets, interests, purchase activity, and medical history) is not exposed. 
In addition to protecting individual privacy, management costs necessitate aggregating data by specific units (e.g., cities, branches, hospitals, and schools). 
To analyze regional trends, the latitude and longitude pair for each unit is linked to the corresponding summary statistics. 
Consequently, techniques are required for spatial modeling given only summarized data. 

In this study, we address learning and inference from summarized data in Gaussian process regression. 
\Cref{overview} illustrates the overview of our approach. 
For input approximation using representative features, we demonstrate the theoretical errors in the marginal likelihood and posterior distribution within the Gaussian process regression framework. 
We also propose the concept of sample quasi-likelihood. 
This approach enables the computation of the marginal likelihood and posterior distribution using only summarized data, offering straightforward implementation and low computational complexity. 
The sample quasi-likelihood is defined by a variance function. 
Specifying this variance function allows for capturing non-Gaussian likelihoods that satisfy certain conditions. 
Theoretical and experimental results indicate that the approximation performance strongly depends on the granularity of summarized data relative to the length scale of covariance functions. 
Experiments using real-world spatial data demonstrate the usability of our approach for supervised learning tasks when only summarized data is available.

\section{Related Work} 
There are existing methods capable of handling aggregated outputs. 
Composite likelihood is derived by multiplying a collection of component likelihoods \citep{Besag, Cristiano}. 
This approach is advantageous for dealing with aggregated outputs while retaining some properties of the full likelihood. 
Additionally, synthetic likelihood \citep{Simon, Price} and approximate Bayesian computation \citep{Beaumont} are also promising. 
These methods approximate the posterior distribution by comparing summary statistics with simulated data generated from a tractable probability model. 
Outputs sometimes take distributional forms, such as random lists, intervals, and histograms. 
\citeauthor{Boris} \citeyear{Boris} constructed likelihood functions for these forms. 
Considering that a Gaussian process has parameters corresponding to its inputs, handling outputs not directly linked to inputs becomes crucial. 
Some studies \citep{Chung, Yusuke} have provided methods for learning and inference from aggregated outputs in Gaussian process regression, assuming that input data points or their distribution are available.

\section{Preliminaries} 

\subsection{Notations} 
\Cref{Description} provides the list of symbols. 
Let $\mathcal{X} \subset \mathbb{R}^d$ denote the domain of inputs, and $\mathcal{Y}$ denote the domain of outputs. 
The symbol of $\bm{K_\mathrm{\cdot \cdot}}$ is a Gram matrix. 
The $(i, j)$-th element of the Gram matrix for $(\bm{A}, \bm{B})$ is the return value of a covariance function when the inputs are the $i$-th row of $\bm{A}$ and the $j$-th row of $\bm{B}$. 
The matrix with the opposite order of indices is the transposed matrix (e.g., $\bm{K_\mathrm{*f}} = \bm{K_\mathrm{f*}}^\top$).  
Let $\mathcal{N}$ denote the probability density function of Gaussian distribution, 
the vertical line with equality denote substitution (e.g., $p(\bm{y} \mid \bm{f})|_{\bm{f} = \bm{W_\mathrm{fu}} \bm{u}}$), 
$O$ denote the Landau symbol, 
and $F$ denote the upper cumulative distribution function of the chi-square distribution given by 
\begin{align} 
  F (a, b) \equiv \Bigl(\int_{0}^{\infty} t^{\frac{a}{2} - 1} \exp(- t) dt \Bigr)^{-1} \int_{\frac{b}{2}}^{\infty} t^{\frac{a}{2} - 1} \exp(- t) dt \nonumber 
\end{align}   
for any $a \in \mathbb{N}$ and $b \in [0, \infty)$.  
The Landau symbol in matrix operations is applied to each matrix element. 
We assume 
that symmetric Gram matrices are positive-definite, 
that likelihood functions and their gradients are differentiable and bounded, 
and that covariance functions are symmetric and continuous.

\begin{table*}[t]
  \centering
  \begin{threeparttable}[]
  \begin{tabular}{cccc} 
    \toprule
    Function Name & $k(\bm{x}, \bm{x^\prime})$  &  $\zeta_1 (\alpha, \bm{z}, \bm{z^\prime})$ &  $\zeta_2 (\alpha, \bm{z}, \bm{x^\mathrm{*}})$ \\ 	
    \midrule 
    Laplacian   
    & $\exp (- |\bm{x} - \bm{x^\prime}|)$ 
    & $1 - \exp (- 2 \alpha)$ 
    & $1 - \exp (- \alpha)$ \\ 
    Gaussian  
    & $\exp (- \frac{1}{2} |\bm{x} - \bm{x^\prime}|^2)$ 
    & $1 - \exp \bigl(- 2 \alpha (|\bm{z} - \bm{z^\prime}| + \alpha) \bigr)$ 
    & $1 - \exp \bigl(- \alpha (|\bm{z} - \bm{x^\mathrm{*}}| + \frac{3}{2} \alpha) \bigr)$   \\ 
    \bottomrule 
  \end{tabular}
  \end{threeparttable} 
  \caption{
  Errors of covariance functions. 
  Given $\bm{z}, \bm{z^\prime}, \bm{x^\mathrm{*}} \in \mathcal{X}$,  
  for any $\alpha \in (0, \infty)$ and $\bm{x}, \bm{x^\prime} \in \mathcal{X}$, 
  $\max\{|\bm{x} - \bm{z}|, |\bm{x^\prime} - \bm{z^\prime}|\} < \alpha \Rightarrow |k(\bm{x}, \bm{x^\prime}) - k(\bm{z}, \bm{z^\prime})| < \zeta_1 (\alpha, \bm{z}, \bm{z^\prime})$ and $|k(\bm{x}, \bm{x}^\mathrm{*}) - k(\bm{z}, \bm{x^\mathrm{*}})| < \zeta_2 (\alpha, \bm{z}, \bm{x^\mathrm{*}})$. 
  See \cref{example_covariance}. 
  }
  \label{example_kernel}
\end{table*}

\subsection{Gaussian Process Regression} 
In this subsection, we describe the basic framework of Gaussian process regression. 
A more detailed introduction is presented in \citep{Carl}. 

A Gaussian process $f \sim \mathcal{GP}(\tau(\cdot), k(\cdot, \cdot))$ is a distribution over functions characterized by a mean function $\tau : \mathcal{X} \rightarrow \mathbb{R}$ and a covariance function $k : \mathcal{X} \times \mathcal{X} \rightarrow (0, \infty)$. 
A stochastic process $\{f(\bm{x}) \mid \bm{x} \in \mathcal{X}\}$ is a Gaussian process if and only if the random variables $\{f(\bm{x}) \mid \bm{x} \in \mathcal{X}^\prime\}$ for any finite set $\mathcal{X}^\prime \subseteq \mathcal{X}$ follow a multivariate normal distribution. 
For simplicity, we take $\tau(\cdot) = 0$. 
Given inputs $\bm{X} \equiv (\bm{x}_i)_{i=1}^n$ with $\bm{x}_i \in \mathcal{X}$ and outputs $\bm{y} \equiv (y_i)_{i=1}^n$ with $y_i \in \mathcal{Y}$, 
the hyperparameters of the covariance function are learned to maximize the log marginal likelihood defined as 
\begin{align} 
    \label{LML}
    \mathcal{L} \equiv \log \int_{\mathbb{R}^n} p(\bm{y} \mid \bm{f}) p(\bm{f}) d \bm{f},   
\end{align} 
where $\bm{f} \equiv (f(\bm{x}_i))_{i=1}^n$, $p(\bm{f}) \equiv \mathcal{N} (\bm{f}; \bm{0}, \bm{K_\mathrm{ff}})$, $p(\bm{y} \mid \bm{f})$ is a likelihood function of $\bm{f}$, and $\bm{K_\mathrm{ff}}$ is the Gram matrix of $(\bm{X}, \bm{X})$. 
The complexity of optimizing $\mathcal{L}$ given a Gaussian likelihood is $O(n^3)$. 
For new inputs $\bm{X_\mathrm{*}} \equiv (\bm{x}_i^\mathrm{*})_{i=1}^{n_\mathrm{*}}$ with $\bm{x}_i^\mathrm{*} \in \mathcal{X}$, the posterior distribution of $\bm{f_\mathrm{*}} \equiv (f(\bm{x}_i^\mathrm{*}))_{i=1}^{n_\mathrm{*}}$ is 
\begin{align} 
    \label{posterior}
    p(\bm{f_\mathrm{*}} \mid \bm{y}) \equiv \exp(- \mathcal{L}) \int_{\mathbb{R}^n} p(\bm{f_\mathrm{*}} \mid \bm{f}) p(\bm{y} \mid \bm{f}) p(\bm{f}) d \bm{f}, 
\end{align} 
where 
$\bm{K_\mathrm{f*}}$ is the Gram matrix of $(\bm{X}, \bm{X_\mathrm{*}})$, 
$\bm{K_\mathrm{**}}$ is that of $(\bm{X_\mathrm{*}}, \bm{X_\mathrm{*}})$,  
and $p(\bm{f_\mathrm{*}} \mid \bm{f}) \equiv \mathcal{N} (\bm{f_\mathrm{*}}; \bm{K_\mathrm{*f}} \bm{K_\mathrm{ff}}^{-1} \bm{f}, \bm{K_\mathrm{**}} - \bm{K_\mathrm{*f}} \bm{K_\mathrm{ff}}^{-1} \bm{K_\mathrm{f*}})$.

\section{Input Approximation} 
To clarify the discussion in this study, we define summarized data and the error $\beta \in (0, \infty)$ of a covariance function. 
Assume that data points $(\bm{x}_i, y_i)_{i=1}^n$ are associated with assignments $\bm{\omega} \equiv (\omega_i)_{i=1}^n$ with $\omega_i \in \{1, \cdots, m\}$. 
Let summarized data be represented as a set of tuples $(\bm{z}_j, \bar{y}_j, n_j)_{j=1}^m$, 
where 
each $\bm{z}_j \in \mathcal{X}$ is a representative feature, 
each $\bar{y}_j \in \mathbb{R}$ is a summary statistic that depends only on the outputs $y_i$ with $\omega_i = j$, 
and each $n_j > 0$ is the number of indices such that $\omega_i = j$.  
Let $\beta$ satisfy the following conditions: 
\begin{itemize}
  \item $|k(\bm{x}_i, \bm{x}_j) - k(\bm{z}_{\omega_i}, \bm{z}_{\omega_j})| < \beta$ for all $1 \leq i \leq j \leq n$. 
  \item $|k(\bm{x}_i, \bm{x}_j^\mathrm{*}) - k(\bm{z}_{\omega_i}, \bm{x}_j^\mathrm{*})| < \beta$ for all $1 \leq i \leq n$ and $1 \leq j \leq n_\mathrm{*}$. 
\end{itemize}     

Even in situations where only summarized data is available, the range of complete data inputs is often accessible. 
This range allows us to evaluate the approximation errors of covariance functions when the inputs are representative features, as described in \cref{example_kernel}.  
For Gaussian process regression, we can flexibly design many types of covariance functions (e.g., covariance functions equivalent to an infinitely wide deep network \citep{Jaehoon} or relying on non-Euclidean metric \citep{Aasa}). 
\Cref{kernel1} gurantees that a small range improves the approximation accuracy for any covariance function. 
\begin{lemma}
    \label{kernel1} 
    Given $\bm{z}, \bm{z^\prime}, \bm{x^\mathrm{*}} \in \mathcal{X}$, 
    there exists $\alpha \in (0, \infty)$ such that 
    $\max\{|\bm{x} - \bm{z}|, |\bm{x^\prime} - \bm{z^\prime}|\} < \alpha \Rightarrow \max\{|k(\bm{x}, \bm{x^\prime}) - k(\bm{z}, \bm{z^\prime})|, |k(\bm{x}, \bm{x}^\mathrm{*}) - k(\bm{z}, \bm{x^\mathrm{*}})|\} < \beta$,  
    for any $\beta \in (0, \infty)$ and $\bm{x}, \bm{x^\prime} \in \mathcal{X}$. 
\end{lemma} 
\begin{proof}
    See \cref{kernel_proof}. 
\end{proof}

Here we approximate the marginal likelihood and posterior distribution using representative features, and evaluate the approximation errors. 
Initially, we consider the case where the inputs are given by $(\bm{z}_{\omega_i})_{i=1}^n$ and the parameters corresponding to the centroids $\bm{Z} \equiv (\bm{z}_i)_{i=1}^m$ are represented by $\bm{u} \equiv (f(\bm{z}_i))_{i=1}^m$.  
In this case, the prior distribution and the likelihood function are   
$p(\bm{u}) \equiv \mathcal{N} (\bm{u}; \bm{0}, \bm{K_\mathrm{uu}})$ and $p(\bm{y} \mid \bm{W_\mathrm{fu}} \bm{u}) \equiv p(\bm{y} \mid \bm{f})|_{\bm{f} = \bm{W_\mathrm{fu}} \bm{u}}$, 
respectively, where $\bm{K_\mathrm{uu}}$ is the Gram matrix of $(\bm{Z}, \bm{Z})$, and $\bm{W_\mathrm{fu}}$ is $n \times m$ matrix with $[\bm{W_\mathrm{fu}}]_{ij} = 1$ if $\omega_i = j$; $[\bm{W_\mathrm{fu}}]_{ij} = 0$ otherwise.  
Using this prior distribution and likelihood function, the log marginal likelihood is 
\begin{align}
    \label{input}
    \mathcal{E} \equiv \log \int_{\mathbb{R}^m} p(\bm{y} \mid \bm{W_\mathrm{fu}} \bm{u}) p(\bm{u}) d \bm{u}.  
\end{align} 
The posterior distribution of $\bm{f_\mathrm{*}}$ is $p(\bm{f_\mathrm{*}} \mid \bm{y}, \bm{\omega})$ defined as 
\begin{align} 
    \label{input2}
    \exp(-\mathcal{E}) \int_{\mathbb{R}^m} p(\bm{f_\mathrm{*}} \mid \bm{W_\mathrm{fu}} \bm{u}) p(\bm{y} \mid \bm{W_\mathrm{fu}} \bm{u}) p(\bm{u}) d \bm{u},  
\end{align} 
where $p(\bm{f_\mathrm{*}} \mid \bm{W_\mathrm{fu}} \bm{u}) \equiv p(\bm{f_\mathrm{*}} \mid \bm{f})|_{\bm{f} = \bm{W_\mathrm{fu}} \bm{u}}$. 
We evaluate their errors in comparison to Gaussian process regression with complete data. 
Let $p(\bm{f} \mid \bm{u})$ be defined as 
\begin{align} 
  \mathcal{N} (\bm{f}; \bm{K_\mathrm{fu}} {\bm{K_\mathrm{uu}}}^{-1} \bm{u}, \bm{K_\mathrm{ff}} - \bm{K_\mathrm{fu}} \bm{K_\mathrm{uu}}^{-1} \bm{K_\mathrm{uf}}), 
\end{align} 
where $\bm{K_\mathrm{fu}}$ is the Gram matrix of $(\bm{X}, \bm{Z})$. 
As in \cref{LML} and \cref{posterior}, the marginal likelihood and posterior distribution of the original model contain $p(\bm{f}) = \int_{\mathbb{R}^m} p(\bm{f} \mid \bm{u}) p(\bm{u}) d \bm{u}$. 
Considering that \cref{input} and \cref{input2} correspond to the case where $\bm{f} = \bm{W_\mathrm{fu}} \bm{u}$ holds, we proceed to analyze the integral of $p(\bm{f} \mid \bm{u})$. 
To facilitate this analysis, we introduce the following lemma. 
\begin{lemma}
    \label{kernel2} 
    Let $\gamma$ denote the maximum absolute value of the elements in $\bm{W_\mathrm{fu}} - \bm{K_\mathrm{fu}} {\bm{K_\mathrm{uu}}}^{-1}$. 
    Then we have 
    \begin{align}
    \label{W_KK} 
    \bm{K_\mathrm{ff}} - \bm{K_\mathrm{fu}} \bm{K_\mathrm{uu}}^{-1} \bm{K_\mathrm{uf}} = O(\beta + m \beta \gamma).   
    \end{align}   
\end{lemma} 
\begin{proof}
    See \cref{kernel2_proof}. 
\end{proof} 
The error $\gamma$ tends to become small when $\beta$ is sufficiently small. 
In such cases, from \cref{kernel1} and \cref{kernel2}, the integral of $p(\bm{f} \mid \bm{u})$ over the region within a certain distance from the point $\bm{f} = \bm{W_\mathrm{fu}} \bm{u}$ approaches $1$ as the range of inputs becomes small. 
The following lemmas evaluate this dynamic. 
For simplicity, we assume $\gamma \neq 0$. 
\begin{lemma}   
  \label{lem_bound}
  For $\delta_1 \in [0, \infty)$ and $\delta_2 \in [0, \delta_1]$, define   
  \begin{align}
    \label{epsi}
    \epsilon(\delta_1, \delta_2) \equiv F \Bigl(m, \frac{\kappa (\delta_1 - \delta_2)^2}{\lambda_1} m\Bigr) + F \Bigl(n, \frac{\delta_2^2}{\lambda_2}n\Bigr),    
  \end{align} 
  where 
  $\kappa \equiv \inf_{\bm{u} \in \mathbb{R}^m; \lvert(\bm{W_\mathrm{fu}} - \bm{K_\mathrm{fu}} {\bm{K_\mathrm{uu}}}^{-1}) \bm{u}\rvert = 1} \frac{n}{m} \lvert\bm{u}\rvert^2$, 
  $\lambda_1$ is the maximum eigenvalue of $\bm{K_\mathrm{uu}}$, 
  and $\lambda_2$ is that of $\bm{K_\mathrm{ff}} - \bm{K_\mathrm{fu}} \bm{K_\mathrm{uu}}^{-1} \bm{K_\mathrm{uf}}$.  
  Then we have  
  \begin{align} 
  \label{error_bound}
  \int_{\mathbb{R}^m} \int_{\mathbb{R}^n \setminus R(\bm{u}, \delta_1)} p(\bm{f} \mid \bm{u}) p(\bm{u}) d\bm{f} d\bm{u} \leq \epsilon(\delta_1, \delta_2),  
  \end{align} 
  where $R(\bm{u}, \delta_1) \equiv \{\bm{f} \in \mathbb{R}^n \mid \lvert\bm{f} - \bm{W_\mathrm{fu}} \bm{u}\rvert \leq \sqrt{n} \delta_1\}$.    
\end{lemma}
\begin{proof}
  See \cref{proof_image} and \cref{proof_lem_bound}. 
\end{proof}
\begin{lemma}   
  \label{convert} 
  The following holds: 
  \begin{align} 
    \kappa^{-1} = O(m \gamma^2),~~\lambda_2 = O(\beta + m \beta \gamma). 
  \end{align} 
\end{lemma}   
\begin{proof}
  See \cref{convert_proof}. 
\end{proof} 
\begin{lemma}   
  \label{monotonically} 
  Suppose $\xi \geq \xi_0$, 
  where $\xi_0$ is the larger value satisfying $2 \sqrt{\pi \xi_0} \exp(- \xi_0) = 1$. 
  Then $F(m, \xi m)$ is monotonically decreasing with respect to $m$.   
\end{lemma}   
\begin{proof}
  See \cref{proof_monotonically}. 
\end{proof} 
The percentiles of the chi-square distributions appearing in \cref{epsi} depend on $\kappa$ and $\lambda_2$. 
\Cref{convert} shows that the integral in \cref{error_bound} approaches zero as $\beta$ and $\gamma$ decrease. 
The evaluation of the integral behaves oppositely with respect to $m$ in \cref{convert} and \cref{monotonically}. 
While the integral decreases as $\delta_1$ increases, $\bm{f}$ within $R(\bm{u}, \delta_1)$ move away from $\bm{W_\mathrm{fu}} \bm{u}$. 
Considering this behavior, we derive the error bounds for the marginal likelihood and posterior distribution. 
\begin{theorem} 
    \label{consistency}
    Let $\eta$ be defined as  
    \begin{align}
      \eta \equiv \inf_{\xi \in [\xi_0, \infty)} \bigl(\sqrt{\xi \lambda_1 \kappa^{-1}} + \sqrt{\xi \lambda_2} + F(m, \xi m)\bigr). 
    \end{align} 
    Given the hyperparameters and complete data, we have   
    \begin{align}
    \label{eta} 
    \mathcal{L} - \mathcal{E} = O(\eta), \\ 
    \label{eta2} 
    \bigl\lVert\mathbb{E}_{p(\bm{f_\mathrm{*}} \mid \bm{y})} [\bm{f_\mathrm{*}}] - \mathbb{E}_{p(\bm{f_\mathrm{*}} \mid \bm{y}, \bm{\omega})} [\bm{f_\mathrm{*}}]\bigr\rVert = O(\eta),    
    \end{align} 
    where $\lVert\cdot\rVert$ denote the norm of a vector. 
\end{theorem}
\begin{proof}
  See \cref{bound_AP}. 
\end{proof} 

\begin{figure}[t] 
  \begin{center}
  \includegraphics[width=82.5mm]{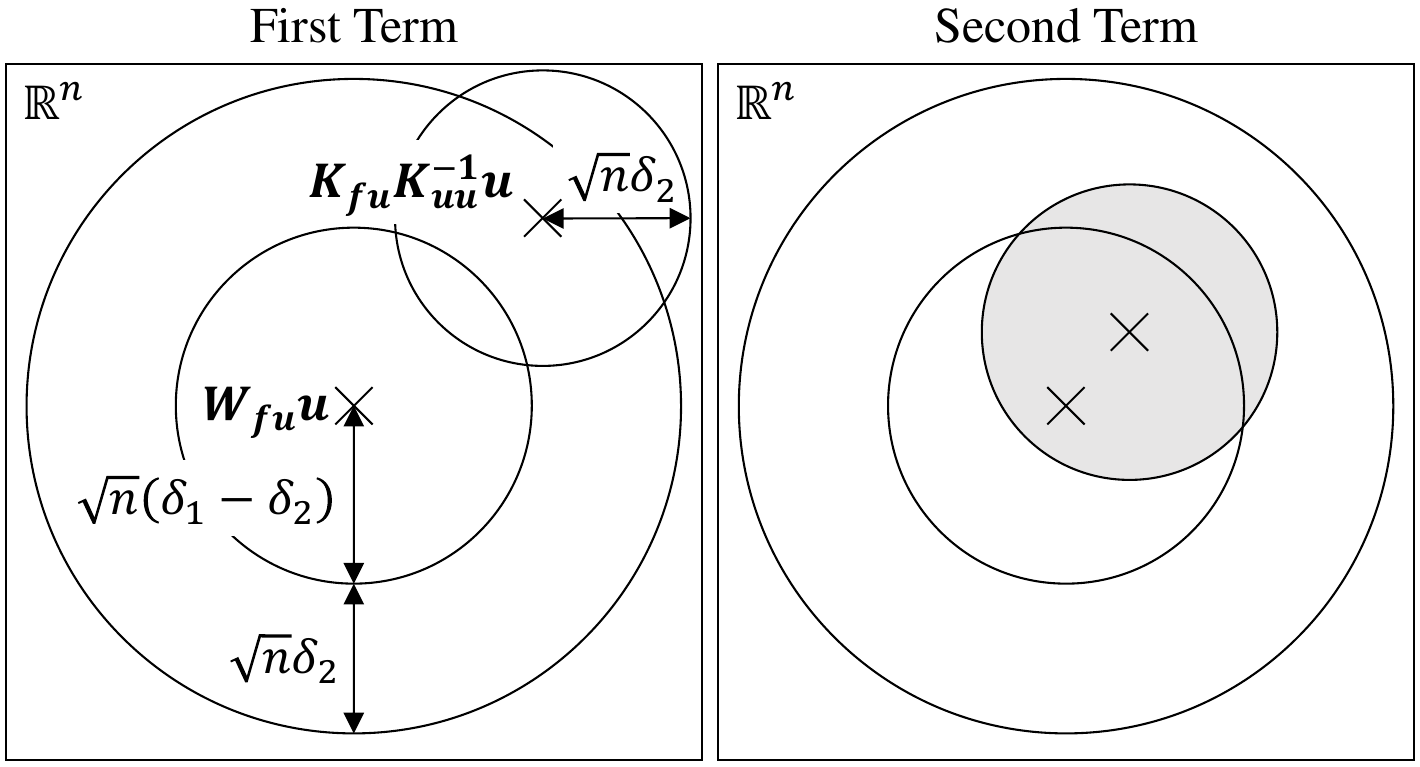}
  \end{center}
\caption{
  Sketch of the proof of \cref{lem_bound}. 
  The first and second terms in \cref{epsi} correspond respectively to the left and right sides.    
  Evaluating the integral in \cref{error_bound} requires analyzing $p(\bm{f} \mid \bm{u})$ over $\mathbb{R}^n \setminus R(\bm{u}, \delta_1)$, posing a significant challenge.   
  Consequently, to assess the upper bound, we consider the hypersphere centered at the mean vector in $p(\bm{f} \mid \bm{u})$.  
  The space of $\bm{u}$ is divided based on whether $R(\bm{u}, \delta_1)$ encompasses the hypersphere. 
  Gray indicates the space that is not integrated with respect to $\bm{f}$. 
}
\label{proof_image}
\end{figure} 

\Cref{consistency} demonstrates that the performance of input approximation improves as $\eta$ decreases.  
While the first and second terms in $\eta$ increase as $\xi$ increases, the last term decreases. 
\Cref{convert} suggests that smaller values of $\beta$ and $\gamma$, relative to the square root of $m$, help prevent the first and second terms from increasing. 
From \cref{monotonically}, the last term decreases as the summarized data becomes finer. 
\Cref{example_inputs} shows that $\eta$ decreases with finer summarized data. 
Furthermore, the length scale of the covariance functions significantly affects $\eta$. 
As the length scale increases, $\eta$ decreases due to smaller $\beta$ and $\gamma$. 
Note that this result does not directly prove the effect of the length scale on input approximation, as the proportional constants in \cref{eta} and \cref{eta2} contain the hyperparameters.

\begin{figure}[t]
  \begin{center}
  \includegraphics[width=82.5mm]{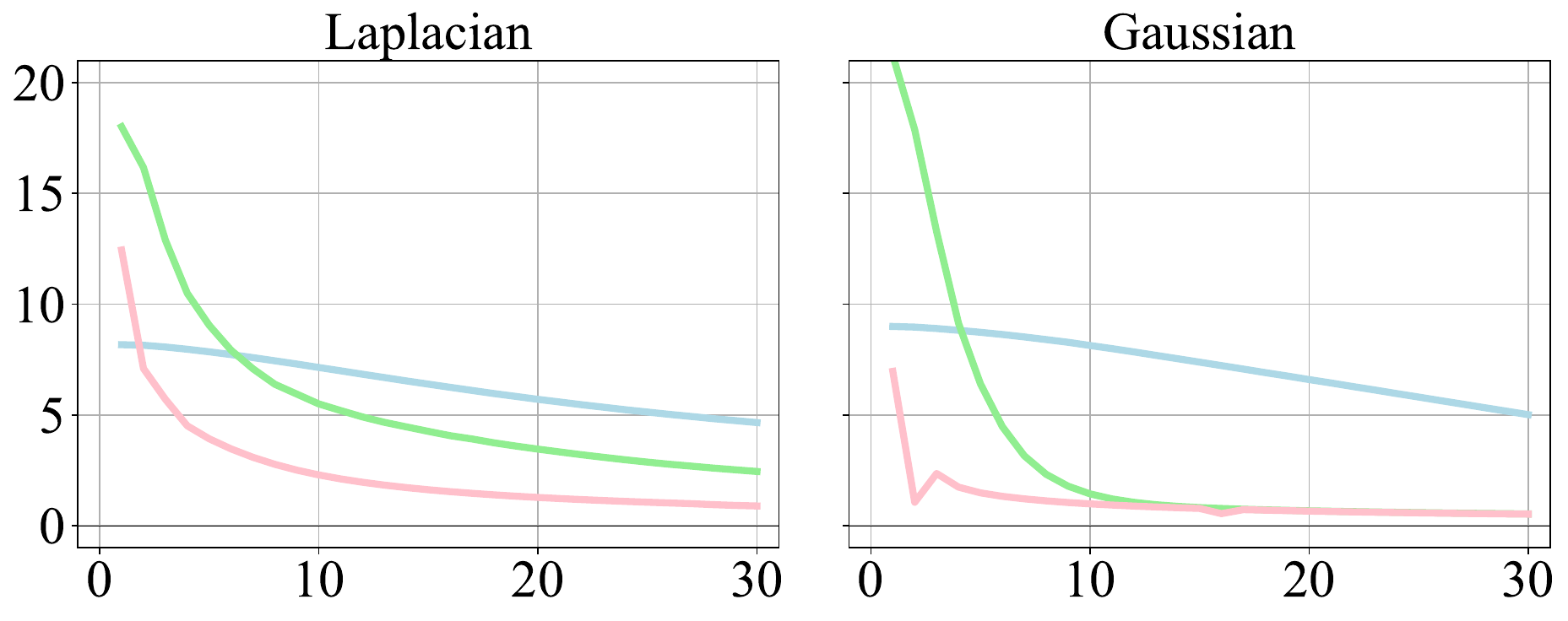}
  \end{center}
\caption{
Behavior of $\eta$ in a toy model. 
Let $\theta$ represent the length scale. 
Light blue: $\theta = 0.1$. Light green: $\theta = 1$. Pink: $\theta = 10$. 
The left and right figures correspond to the covariance functions $\exp (- \frac{1}{\theta} |x - x^\prime|)$ and $\exp (- \frac{1}{2 \theta^2} |x - x^\prime|^2)$, respectively.   
In each figure, the vertical and horizontal axes represent $\eta$ and $m$, respectively. 
The $n = 1000$ inputs and $m$ representative features are equally spaced within $[0, 2 \pi]$. 
The assignment $\omega_i$ of a data point $x_i$ is the index $j$ of the closest centroid $z_j$.   
\Cref{optimization_toy} explains the process of obtaining $\kappa$ and $\eta$.  
}
\label{example_inputs}
\end{figure}

\section{Sample Quasi-Likelihood} 
\label{SQL}
Since the likelihood function $p(\bm{y} \mid \bm{W_\mathrm{fu}} \bm{u})$ requires the outputs of complete data,  
the log marginal likelihood $\mathcal{E}$ and the posterior distribution $p(\bm{f_\mathrm{*}} \mid \bm{y}, \bm{\omega})$ still cannot be computed using only summarized data. 
Consequently, we replace the likelihood function with a function that excludes complete data outputs and incorporates summary statistics $\bm{\bar{y}} \equiv (\bar{y}_i)_{i=1}^m$. 
As one such function, we propose the concept of sample quasi-likelihood characterized by a variance function $v : \mathbb{R} \rightarrow (0, \infty)$ as follows: 
\begin{definition} 
    A sample quasi-likelihood function $\bar{Q} : \mathbb{R}^m \times \mathbb{R}^m \rightarrow \mathbb{R}$ is a function such that 
    \begin{align} 
    \frac{\partial \bar{Q}(\bm{\bar{y}}, \bm{u})}{\partial \bm{u}} = \bm{V_\mathrm{uu}}^{-1} (\bm{\bar{y}} - \bm{u}), 
    \end{align}   
    where $\bm{V_\mathrm{uu}} \equiv \diag(n_1^{-1} v(\bar{y}_1), \cdots, n_m^{-1} v(\bar{y}_m))$. 
\end{definition} 
Suppose that the prior distribution is $p(\bm{u})$, and the likelihood function is $\mathcal{N} (\bm{\bar{y}}; \bm{u}, \bm{V_\mathrm{uu}})$. 
Then the log marginal likelihood is  
\begin{align} 
\mathcal{Q} \equiv &- \frac{m}{2} \log (2 \pi) - \frac{1}{2} \log \lvert \bm{K_\mathrm{uu}} + \bm{V_\mathrm{uu}} \rvert&  \nonumber \\ 
                    &- \frac{1}{2} {\bm{\bar{y}}}^{\top} (\bm{K_\mathrm{uu}} + \bm{V_\mathrm{uu}})^{-1} \bm{\bar{y}}.& 
\end{align} 
The posterior distribution of $\bm{f_\mathrm{*}}$ is $\mathcal{N} (\bm{f_\mathrm{*}}; \bm{\mu_q}, \bm{\Sigma_q})$, 
where 
$\bm{\mu_q} \equiv \bm{K_\mathrm{*u}} (\bm{K_\mathrm{uu}} + \bm{V_\mathrm{uu}})^{-1} \bm{\bar{y}}$,  
$\bm{\Sigma_q} \equiv \bm{K_\mathrm{**}} - \bm{K_\mathrm{*u}} (\bm{K_\mathrm{uu}} + \bm{V_\mathrm{uu}})^{-1} \bm{K_\mathrm{u*}}$,  
and $\bm{K_\mathrm{u*}}$ is the Gram matrix of $(\bm{Z}, \bm{X_\mathrm{*}})$. 
See \cref{Lap_likelihood_quasi}. 
These marginal likelihood and posterior distribution are computed using only summarized data. 
The implementation is straightforward, as it is equivalent to Gaussian process regression with a Gaussian likelihood. 
The computational complexity is $O(m^3)$. 

Despite these advantages, we still have not discussed the difference between likelihood and sample quasi-likelihood.
Here we demonstrate the asymptotic behavior of the sample quasi-likelihood. 
The assumptions outlined below enable us to apply the Laplace approximation using summary statistics. 
Note that the asymptotic behavior in Bayesian statistics can be found in \citep{Watanabe}. 
\begin{assumption} 
    \label{assumption_likelihood}
    $- \nabla \nabla \log p(\bm{y} \mid \bm{W_\mathrm{fu}} \bm{u}) |_{\bm{u} = \bm{\bar{y}}} = \bm{V_\mathrm{uu}}^{-1}$. 
\end{assumption}
\begin{assumption}     
    \label{assumption_MLE}
    $p(\bm{y} \mid \bm{W_\mathrm{fu}} \bm{u})$ is unimodal with respect to $\bm{u}$, having its mode at $\bm{\bar{y}}$. 
\end{assumption} 
The Laplace approximation is conventionally used to handle non-Gaussian likelihoods by approximating the posterior with a Gaussian centered at the maximum a posteriori estimate. 
Unlike this, our approach replaces the likelihood with a Gaussian centered at the maximum likelihood estimate. 
\begin{theorem}
\label{Lap_theorem}
  Suppose that \cref{assumption_likelihood} and \cref{assumption_MLE} hold. 
  Then we have  
  \begin{align}
  \label{E_likelihood}
  \mathcal{E} - \mathcal{Q} = &\log p(\bm{y} \mid \bm{W_\mathrm{fu}} \bm{u}) |_{\bm{u} = \bm{\bar{y}}} + \frac{m}{2} \log (2 \pi)& \nonumber \\  
                              &+ \frac{1}{2} \log \lvert \bm{V_\mathrm{uu}}\rvert + o_p (m),& 
  \end{align} 
  where $o_p$ denote convergence in probability. 
  Additionally, the posterior distribution $p(\bm{f_\mathrm{*}} \mid \bm{y}, \bm{\omega})$ asymptotically converges to $\mathcal{N} (\bm{f_\mathrm{*}}; \bm{\mu_p}, \bm{\Sigma_p})$,  
  where 
  \begin{align}
  \bm{\mu_p}    &\equiv \bm{K_\mathrm{*f}} \bm{K_\mathrm{ff}}^{-1} \bm{W_\mathrm{fu}} (\bm{V_\mathrm{uu}}^{-1} + \bm{K_\mathrm{uu}}^{-1})^{-1} \bm{V_\mathrm{uu}}^{-1} \bm{\bar{y}},&  \nonumber \\ 
  \bm{\Sigma_p} &\equiv \bm{K_\mathrm{**}} - \bm{K_\mathrm{*f}} \bm{K_\mathrm{ff}}^{-1} \bm{K_\mathrm{f*}}&  \nonumber  \\ 
  &+ \bm{K_\mathrm{*f}} \bm{K_\mathrm{ff}}^{-1} \bm{W_\mathrm{fu}} (\bm{V_\mathrm{uu}}^{-1} + {\bm{K_\mathrm{uu}}}^{-1})^{-1}  \bm{W_\mathrm{uf}} \bm{K_\mathrm{ff}}^{-1} \bm{K_\mathrm{f*}}.& \nonumber   
  \end{align}   
\end{theorem} 
\begin{proof}
  See \cref{Lap}. 
\end{proof} 
\begin{theorem} 
  \label{approx_matrix}
  Suppose that $(\bm{V_\mathrm{uu}}^{-1} + {\bm{K_\mathrm{uu}}}^{-1})^{-1}, \bm{K_\mathrm{uu}}^{-1}$, $\bm{K_\mathrm{u*}}, \bm{K_\mathrm{ff}}^{-1}, \bm{W_\mathrm{fu}}$, and $\bm{V_\mathrm{uu}}^{-1} \bm{\bar{y}}$ become $O(\beta)$ when multiplied by a matrix of $O(\beta)$.  
  Then we have 
  \begin{align} 
  \label{posterior_mu_sigma}
  \bm{\mu_p} - \bm{\mu_q} = O(\beta),~~\bm{\Sigma_p} - \bm{\Sigma_q} = O(\beta + m \beta^2). 
  \end{align}   
\end{theorem} 
\begin{proof}
    See \cref{proof_approx_matrix}. 
\end{proof} 
\Cref{Lap_theorem} describes the asymptotic behavior of the marginal likelihood and posterior distribution. 
This approximation performs well when $n$ is sufficiently larger than $m$.  
Since the right side of \cref{E_likelihood} does not depend on the hyperparameters of a covariance function, we can employ $\mathcal{Q}$ to optimize $\mathcal{E}$.  
Regarding the posterior distribution, we cannot compute $\mathcal{N} (\bm{f_\mathrm{*}}; \bm{\mu_p}, \bm{\Sigma_p})$ given only summarized data since it contains $\bm{K_\mathrm{ff}}$ and $\bm{K_\mathrm{*f}}$.  
\Cref{approx_matrix} allows us to avoid the computation of them.  
From \cref{posterior_mu_sigma}, the sample quasi-likelihood can be used for a smaller range of inputs. 
\Cref{example_outputs} illustrates that the approximation errors of the marginal likelihood and posterior distribution are affected by the relative granularity of summarized data with respect to the length scale of the covariance functions. 
Their behavior largely aligns with the dynamics of $\eta$ in \cref{example_inputs}.  
While the input domain of this toy model is one-dimensional, this behavior is expected to persist in higher-dimensional data. 
However, it is important to note that the range of inputs expands as the input dimension increases.

\begin{figure}[t]
  \begin{center}
      \includegraphics[width=82.5mm]{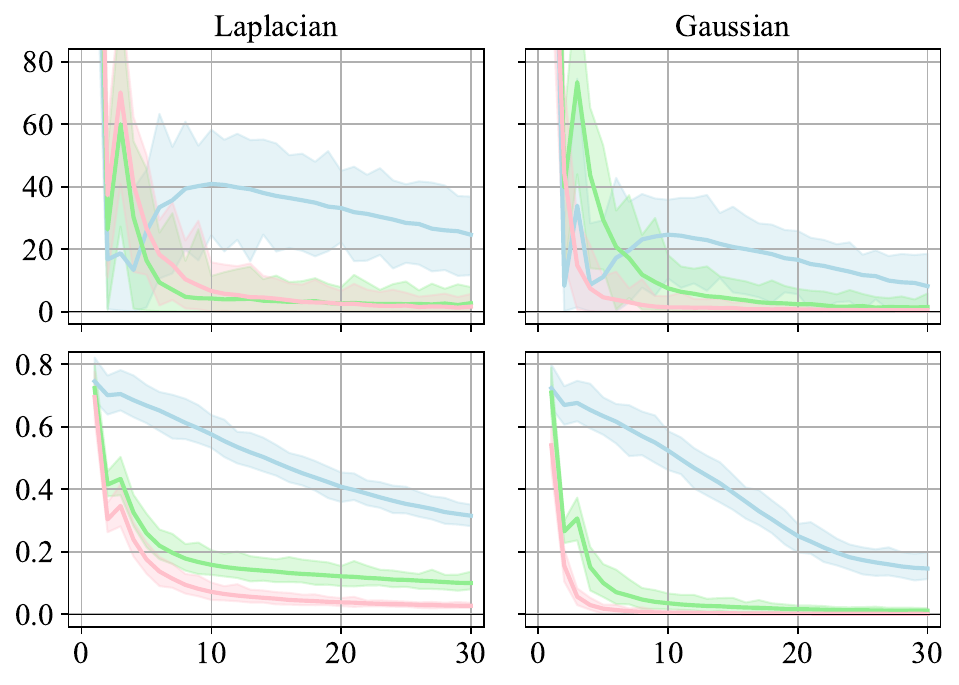}
  \end{center}
  \caption{
  Behavior of marginal likelihood and posterior distribution in a toy model. 
  The vertical axis in the upper figure represents the absolute difference between $\mathcal{Q} + \log p(\bm{y} \mid \bm{W_\mathrm{fu}} \bm{u}) |_{\bm{u} = \bm{\bar{y}}} + \frac{m}{2} \log (2 \pi) + \frac{1}{2} \log \lvert \bm{V_\mathrm{uu}}\rvert$ and $\mathcal{L}$.   
  The vertical axis in the lower figure represents the root mean squared error (RMSE) between $\bm{K_\mathrm{*f}} (\bm{K_\mathrm{ff}} + \diag(1, \cdots, 1))^{-1} \bm{y}$ and $\bm{\mu_q}$. 
  Each solid line represents the average over $100$ trials, with the shaded region in each color showing the range between the maximum and minimum values. 
  The $n_\mathrm{*} = 1000$ new inputs were uniformly generated within $[0, 2 \pi]$.  
  Each output $y_i$ was generated from $\mathcal{N} (\sin(\bm{x}_i), 1)$. 
  The summary statistic is the sample mean. 
  The likelihood function is $\prod_{i=1}^n \mathcal{N} (y_i; f(\bm{x}_i), 1)$. 
  Other displays and conditions are identical to those in \cref{example_inputs}. 
  }
  \label{example_outputs}
\end{figure} 

\begin{table*}[t]
  \centering
  \begin{threeparttable}[]
  \begin{tabular}{ccccccccc} 
  \toprule
  \multirow{2.5}{*}{Output} & \multirow{2.5}{*}{Covariance}  & \multicolumn{6}{c}{Grid Size}  \\ 	
  \cmidrule(lr){3-8} &  & $1.6 \times 1.6$ & $0.8 \times 0.8$ & $0.4 \times 0.4$ & $0.2 \times 0.2$ & $0.1 \times 0.1$ & $0.05 \times 0.05$  \\ 
  \midrule 
  MedInc & Laplacian & $0.57 \pm 0.04$ & $0.60 \pm 0.05$ & $0.59 \pm 0.05$ & $0.54 \pm 0.05$ & $0.44 \pm 0.05$ & $0.32 \pm 0.05$ \\ 
  MedInc & Gaussian & $0.56 \pm 0.09$ & $0.59 \pm 0.11$ & $0.57 \pm 0.12$ & $0.50 \pm 0.14$ & $0.39 \pm 0.12$ & $0.28 \pm 0.10$ \\ 
  \midrule 
  HouseAge & Laplacian & $0.63 \pm 0.03$ & $0.57 \pm 0.03$ & $0.52 \pm 0.04$ & $0.41 \pm 0.04$ & $0.30 \pm 0.04$ & $0.19 \pm 0.03$ \\ 
  HouseAge & Gaussian & $0.62 \pm 0.04$ & $0.57 \pm 0.04$ & $0.54 \pm 0.05$ & $0.40 \pm 0.06$ & $0.28 \pm 0.07$ & $0.17 \pm 0.06$ \\ 
  \midrule 
  AveRooms & Laplacian & $0.62 \pm 0.12$ & $0.59 \pm 0.11$ & $0.54 \pm 0.11$ & $0.49 \pm 0.09$ & $0.40 \pm 0.10$ & $0.29 \pm 0.08$ \\ 
  AveRooms & Gaussian & $0.62 \pm 0.16$ & $0.63 \pm 0.16$ & $0.61 \pm 0.18$ & $0.59 \pm 0.23$ & $0.48 \pm 0.26$ & $0.32 \pm 0.24$ \\ 
  \midrule 
  AveBedrms & Laplacian & $0.62 \pm 0.15$ & $0.58 \pm 0.15$ & $0.51 \pm 0.12$ & $0.42 \pm 0.10$ & $0.32 \pm 0.13$ & $0.22 \pm 0.10$ \\ 
  AveBedrms & Gaussian & $0.77 \pm 0.23$ & $0.77 \pm 0.24$ & $0.75 \pm 0.22$ & $0.71 \pm 0.27$ & $0.59 \pm 0.30$ & $0.44 \pm 0.33$ \\ 
  \midrule 
  Population & Laplacian & $0.26 \pm 0.10$ & $0.26 \pm 0.10$ & $0.25 \pm 0.11$ & $0.25 \pm 0.12$ & $0.22 \pm 0.11$ & $0.19 \pm 0.09$ \\ 
  Population & Gaussian & $0.16 \pm 0.12$ & $0.16 \pm 0.12$ & $0.15 \pm 0.13$ & $0.14 \pm 0.13$ & $0.13 \pm 0.13$ & $0.11 \pm 0.12$ \\ 
  \midrule 
  AveOccup & Laplacian & $0.47 \pm 0.17$ & $0.47 \pm 0.18$ & $0.46 \pm 0.18$ & $0.45 \pm 0.17$ & $0.41 \pm 0.16$ & $0.30 \pm 0.12$ \\ 
  AveOccup & Gaussian & $0.46 \pm 0.33$ & $0.46 \pm 0.33$ & $0.45 \pm 0.34$ & $0.45 \pm 0.33$ & $0.43 \pm 0.32$ & $0.36 \pm 0.32$ \\ 
  \midrule 
  MedValue & Laplacian & $0.79 \pm 0.03$ & $0.78 \pm 0.04$ & $0.74 \pm 0.05$ & $0.63 \pm 0.05$ & $0.47 \pm 0.05$ & $0.32 \pm 0.04$ \\ 
  MedValue & Gaussian & $0.85 \pm 0.08$ & $0.84 \pm 0.11$ & $0.75 \pm 0.14$ & $0.58 \pm 0.20$ & $0.41 \pm 0.18$ & $0.27 \pm 0.13$ \\ 
  \bottomrule 
  \end{tabular}
  \end{threeparttable} 
  \caption{
  Approximation performance. 
  To compare our approximation with the original regression, we present the mean and standard deviation of the RMSE between the mean vectors of their posterior distributions over $100$ trials. 
  The RMSE was normalized by dividing it by the standard deviation of the training data outputs. 
  Let the likelihood be Gaussian and let $n = 1000$. 
  }
  \label{performance1}
\end{table*} 

The discussion so far supports the use of the sample quasi-likelihood for learning from summarized data. 
Addressing non-Gaussian likelihoods with summarized data presents a significant challenge, just as it does in Gaussian process regression with complete data. 
Our approach tackles this challenge by specifying an appropriate variance function, ensuring that \cref{assumption_likelihood} is satisfied for a implicit likelihood function. 
Furthermore, the sample quasi-likelihood allows for the use of various summary statistics that satisfy \cref{assumption_MLE}. 
This assumption requires that the summary statistic corresponds to the maximum likelihood estimate of the implicit likelihood function. 
For example, the sample median serves as the maximum likelihood estimate of the location parameter in the Laplace distribution. 
Our approximation performs poorly when the likelihood function associated with the variance function and summary statistic is flat. 
The concept of sample quasi-likelihood is analogous to the quasi-likelihood proposed by \citep{Wedderburn, Wedderburn2}. 
While our approach also benefits from specifying a variance function instead of a likelihood function, its motivation and definition differ slightly. 
The formulation of the sample quasi-likelihood provides a computational advantage by expressing the marginal likelihood and posterior distribution in closed form. 
The comparison between both approaches is presented in \cref{QL_Description}. 

\begin{table*}[t]
  \centering
  \begin{threeparttable}[]
  \begin{tabular}{ccccccccc} 
  \toprule
  \multirow{2.5}{*}{Output} & \multirow{2.5}{*}{Likelihood} & \multicolumn{2}{c}{Complete VI} & \multicolumn{2}{c}{Summarized VI} & \multicolumn{2}{c}{Our Approach} \\ 	
  \cmidrule(lr){3-4} \cmidrule(lr){5-6} \cmidrule(lr){7-8} &  & RMSE & Time & RMSE & Time & RMSE & Time \\ 
  \midrule 
  MedInc & Gaussian & $0.881 \pm 0.014$ & $41 \pm 4$ & $1.015 \pm 0.015$ & $30 \pm 14$ & $\bm{0.970 \pm 0.004}$ & $\bm{6 \pm 1}$ \\ 
  MedInc & Poisson & $0.887 \pm 0.024$ & $59 \pm 14$ & $1.014 \pm 0.016$ & $59 \pm 12$ & $\bm{0.977 \pm 0.008}$ & $\bm{6 \pm 2}$ \\ 
  \midrule 
  HouseAge & Gaussian & $0.880 \pm 0.049$ & $39 \pm 11$ & $\bm{0.970 \pm 0.024}$ & $29 \pm 17$ & $1.000 \pm 0.000$ & $\bm{4 \pm 0}$ \\ 
  HouseAge & Poisson & $0.861 \pm 0.049$ & $55 \pm 18$ & $\bm{0.966 \pm 0.021}$ & $51 \pm 20$ & $1.032 \pm 0.060$ & $\bm{7 \pm 2}$ \\ 
  \midrule 
  AveRooms & Gaussian & $0.988 \pm 0.031$ & $15 \pm 12$ & $0.995 \pm 0.014$ & $12 \pm 12$ & $\bm{0.975 \pm 0.051}$ & $\bm{10 \pm 3}$ \\ 
  AveRooms & Poisson & $0.990 \pm 0.031$ & $19 \pm 19$ & $0.995 \pm 0.015$ & $19 \pm 20$ & $\bm{0.949 \pm 0.012}$ & $\bm{8 \pm 1}$ \\ 
  \midrule 
  AveBedrms & Gaussian & $0.989 \pm 0.029$ & $15 \pm 13$ & $0.996 \pm 0.012$ & $14 \pm 12$ & $\bm{0.964 \pm 0.041}$ & $\bm{12 \pm 3}$ \\ 
  AveBedrms & Poisson & $0.927 \pm 0.021$ & $62 \pm 9$ & $\bm{0.941 \pm 0.016}$ & $61 \pm 11$ & $0.962 \pm 0.044$ & $\bm{10 \pm 3}$ \\ 
  \midrule 
  Population & Gaussian & $1.242 \pm 0.224$ & $5 \pm 3$ & $1.041 \pm 0.085$ & $18 \pm 15$ & $\bm{1.000 \pm 0.000}$ & $\bm{1 \pm 0}$ \\ 
  Population & Poisson & $0.990 \pm 0.009$ & $45 \pm 26$ & $1.000 \pm 0.005$ & $37 \pm 24$ & $1.004 \pm 0.013$ & $\bm{6 \pm 1}$ \\ 
  \midrule 
  AveOccup & Gaussian & $1.002 \pm 0.008$ & $12 \pm 11$ & $1.002 \pm 0.006$ & $10 \pm 10$ & $1.002 \pm 0.006$ & $\bm{6 \pm 2}$ \\ 
  AveOccup & Poisson & $1.021 \pm 0.141$ & $59 \pm 23$ & $1.015 \pm 0.063$ & $43 \pm 15$ & $\bm{1.001 \pm 0.002}$ & $\bm{6 \pm 2}$ \\ 
  \midrule 
  MedValue & Gaussian & $0.699 \pm 0.009$ & $41 \pm 2$ & $1.008 \pm 0.015$ & $40 \pm 2$ & $\bm{0.904 \pm 0.010}$ & $\bm{6 \pm 1}$ \\ 
  MedValue & Poisson & $0.746 \pm 0.033$ & $59 \pm 15$ & $0.973 \pm 0.038$ & $58 \pm 14$ & $\bm{0.909 \pm 0.011}$ & $\bm{5 \pm 0}$ \\ 
  \bottomrule 
  \end{tabular}
  \end{threeparttable} 
  \caption{
  Prediction performance. 
  The mean and standard deviation of the RMSE and execution time [s] over $100$ trials are shown. 
  The RMSE between predictors and test outputs was normalized by dividing it by the standard deviation of the test data outputs. 
  The prediction was obtained by applying the mapping $g$ to the mean vector of the posterior distribution of $\bm{f_\mathrm{*}}$. 
  The execution time was measured from the start of learning to the end of prediction.
  The covariance function was the Gaussian.  
  The grid size was $0.4 \times 0.4$. 
  Let $n = 10000$. 
  The displays are bolded when the mean loss of the summarized VI and our method differs by greater than $0.01$. 
  Similarly, bolding is applied when the difference in the mean execution time exceeds $1$ second. 
  The results of the $0.8 \times 0.8$ grid size can be found in \cref{prediction_performance}. 
  }
  \label{performance2}
\end{table*}

\section{Spatial Modeling} 
\label{Experimental}
We investigate the usage of our method in spatial modeling tasks, using the California housing dataset \footnote{https://lib.stat.cmu.edu/datasets/}. 
Our approach leverages the input approximation via representative features and utilizes the sample quasi-likelihood. 
We evaluate its performance by comparing it with Gaussian process regression given complete data. 
To compute the exact marginal likelihood and posterior distribution, we used a Gaussian likelihood in this comparison. 
We also perform predictive testing with the Gaussian and Poisson likelihoods. 

\textbf{Summarized Data.} 
The dataset was randomly shuffled and then split to obtain the $n$ training data points and $n_{\mathrm{*}} = 20640 - n$ test data points. 
The inputs were the pairs of latitude and longitude included among the attributes.  
The seven remaining attributes were used as outputs, respectively. 
The domain $[32.54, 41.95] \times [- 124.35, - 114.31]$ in latitude and longitude was divided into girds. 
The data points within each grid were summarized, with the centers of the grids serving as the representative features, and the summary statistics being the sample means. 
Assume that the sample variances are also provided and that the distribution of input data points for learning is unavailable. 

\textbf{Functions.} 
The likelihood function is presented in \cref{example_likelihood}. 
Let $g(\bm{x})$ denote $f(\bm{x})$ for the Gaussian, and $\exp(\bm{x})$ for the Poisson. 
We replace the summary statistics $\bm{\bar{y}}$ with $(g^{-1}(\bar{y}_i))_{i=1}^m$. 
This replacement makes the likelihood functions consistent with \cref{assumption_MLE}. 
Similarly, we specified the variance functions based on the likelihoods, ensuring \cref{assumption_likelihood} is satisfied. 
The mean function was $\tau(\cdot) = g^{-1}(\frac{1}{n} \sum_{i=1}^n y_i)$. 
The covariance functions were designed by multiplying a constant kernel with the functions used in \cref{example_inputs} and adding white noise. 

\textbf{Hyperparameters.} 
The hyperparameters of the covariance functions were optimized using the L-BFGS-B algorithm, starting from an initial value of $1$.   
This optimization was implemented with version 1.7.3 of the SciPy software \footnote{https://scipy.org/}, with the parameters set to their default values. 
We used a 64-bit Windows machine with the Intel Core i9-9900K @ $3.60$ GHz and $64$ GB of RAM. 
The code was implemented in the python programming language 3.7.3 version.  
As a hyperparameter, we also optimized the variance $\sigma \in (0, \infty)$ of the Gaussian likelihood. 
In this case, the right side of \cref{E_likelihood} contains the hyperparameter. 
Therefore, we employed $\mathcal{E}$ instead of $\mathcal{Q}$ for the Gaussian likelihood, which can be computed using the sample mean and sample variances.    

\textbf{Baselines.} 
For the prediction test baselines, we adopted the sparse variational inference method \citep{Hensman, Zoubin} implemented in version 1.12.0 of the GPy software \footnote{https://sheffieldml.github.io/GPy/}. 
The parameters in GPy were set to their default values. 
The initial locations of $\lfloor\sqrt{m^3 n^{-1}}\rfloor$ inducing points were randomly allocated from $m$ representative locations. 
Two patterns using variational inference were attempted. 
One involved the standard learning and inference with complete data. 
In the other, the inputs were replaced with representative locations (i.e., $(\bm{z}_{\omega_i})_{i=1}^n$). 
The complete outputs were used in both cases, contrasting with our setting where only summary statistics are available. 
We refer to these as the complete VI and summarized VI, respectively. 
The other settings were identical to those in our approach. 

\textbf{Evaluation.} 
\Cref{performance1} shows that the granularity of summarized data strongly impact the approximation performance of our approach.  
This result is consistent with the theoretical findings and the simple simulation using the toy model. 
The superiority or inferiority of the two covariance functions depended on the output attributes, even though the same inputs were used. 
The comparison between the complete VI and summarized VI in \cref{performance2} suggests a negative impact of the coarse inputs. 
Our approach offered advantages in both loss and execution time compared to the summarized VI. 
For the AveRooms, specifying the variance function corresponding to the Poisson distribution was beneficial. 

\section{Conclusion} 
\label{conclusion}
This study focuses on learning and inference when only summarized data is available. 
The introduction of the sample quasi-likelihood facilitates them, including in cases where the likelihood is non-Gaussian. 
Our approach is straightforward to implement and incurs low computational costs. 
We have demonstrated the approximation errors of the marginal likelihood and posterior distribution, both theoretically and experimentally. 
This analysis highlights when our approach works well. 
In particular, the approximation performance depends on the granularity relative to the length scale of covariance functions.  
Experiments on spatial modeling show the impact of using summarized data instead of complete data. 
The feasibility of our approach was supported by the series of discussions in this study. 

\textbf{Limitations.} 
The input approximation does not perform well when the granularity of summarized data is rough. 
Since our approach relies on the Laplace approximation applied to each summarized group, the prediction performance for non-Gaussian observations deteriorates as the number of data points decreases. 
Further investigation is needed for extending our approach to cases where statistics other than the sample mean are available. 
The proportional constants in \cref{eta} and \cref{eta2} contain the hyperparameters of a covariance function, as shown in \cref{bound1} and \cref{bound2}. 
The theoretical behavior of our approximation concerning the hyperparameters requires further analysis. 

\textbf{Future Works.} 
Spatial modeling is valuable for decision-making in many fields, but a significant concern is protecting privacy. 
While our discussion in this study is primarily from the perspective of a data user, the design of datasets to maintain confidentiality, such as through differential privacy \citep{Dwork}, is also important. 
Our approach potentially contributes to encouraging discussions on the granularity of summarized data, considering the trade-offs between usability and confidentiality. 
Additionally, optimizing representative locations and data point assignments to improve approximation performance is an appealing direction for future work.

\bigskip
\bibliography{AAAI2025_GP}

\newpage 
\appendix
\onecolumn

\section{Notations}
\label{Description}

\begin{table*}[h]
  \centering
  \begin{threeparttable}[]
  \begin{tabular}{c|l} 
    \toprule
    Notation  & \multicolumn{1}{c}{Description}  \\ 	
    \midrule 
    $\bm{X}$ & $(\bm{x}_i)_{i=1}^{n}$ with $\bm{x}_i \in \mathcal{X}$ \\ 
    $\bm{y}$ & $(y_i)_{i=1}^{n}$ with $y_i \in \mathcal{Y}$ \\ 
    $\bm{f}$ & $(f(\bm{x}_i))_{i=1}^{n} \in \mathbb{R}^{n}$ \\ 
    $\bm{X_\mathrm{*}}$ & $(\bm{x}_i^\mathrm{*})_{i=1}^{n_\mathrm{*}}$ with $\bm{x}_i^\mathrm{*} \in \mathcal{X}$ \\ 
    $\bm{f_\mathrm{*}}$ & $(f(\bm{x}_i^\mathrm{*}))_{i=1}^{n_\mathrm{*}} \in \mathbb{R}^{n_\mathrm{*}}$ \\ 
    $\bm{K_\mathrm{ff}}$ & the Gram matrix of $(\bm{X}, \bm{X})$ \\ 
    $\bm{K_\mathrm{f*}}$ & the Gram matrix of $(\bm{X}, \bm{X_\mathrm{*}})$ \\  
    $\bm{K_\mathrm{**}}$ & the Gram matrix of $(\bm{X_\mathrm{*}}, \bm{X_\mathrm{*}})$ \\  
    $p(\bm{f})$ & $\mathcal{N} (\bm{f}; \bm{0}, \bm{K_\mathrm{ff}})$ \\  
    $p(\bm{y} \mid \bm{f})$ & a likelihood function of $\bm{f}$ \\   
    $p(\bm{f_\mathrm{*}} \mid \bm{f})$ & $\mathcal{N} (\bm{f_\mathrm{*}}; \bm{K_\mathrm{*f}} \bm{K_\mathrm{ff}}^{-1} \bm{f}, \bm{K_\mathrm{**}} - \bm{K_\mathrm{*f}} \bm{K_\mathrm{ff}}^{-1} \bm{K_\mathrm{f*}})$ \\ 
    $p(\bm{f_\mathrm{*}} \mid \bm{y})$ & $\exp(- \mathcal{L}) \int_{\mathbb{R}^n} p(\bm{f_\mathrm{*}} \mid \bm{f}) p(\bm{y} \mid \bm{f}) p(\bm{f}) d \bm{f}$ \\  
    $\mathcal{L}$ & $\log \int_{\mathbb{R}^n} p(\bm{y} \mid \bm{f}) p(\bm{f}) d \bm{f}$ \\  
    $\bm{Z}$ & $(\bm{z}_i)_{i=1}^{m}$ with $\bm{z}_i \in \mathcal{X}$ \\   
    $\bm{u}$ & $(f(\bm{z}_i))_{i=1}^m \in \mathbb{R}^{m}$ with $\bm{z}_i \in \mathcal{X}$ \\   
    $\bm{\bar{y}}$ & $(\bar{y}_i)_{i=1}^m$ with $\bar{y}_i \in \mathbb{R}$ depending only on the outputs with $\omega_i = j$ \\  
    $\bm{\omega}$ & $(\omega_i)_{i=1}^n$ with $\omega_i \in \{1, \cdots, m\}$ \\   
    $\bm{K_\mathrm{uu}}$ & the Gram matrix of $(\bm{Z}, \bm{Z})$  \\  
    $\bm{K_\mathrm{fu}}$ & the Gram matrix of $(\bm{X}, \bm{Z})$  \\   
    $\bm{K_\mathrm{u*}}$ & the Gram matrix of $(\bm{Z}, \bm{X_\mathrm{*}})$ \\ 
    $\bm{W_\mathrm{fu}}$ & $n \times m$ matrix with $[\bm{W_\mathrm{fu}}]_{ij} = 1$ if $\omega_i = j$; $[\bm{W_\mathrm{fu}}]_{ij} = 0$ otherwise \\  
    $\bm{\mu_p}$ & $\bm{K_\mathrm{*f}} \bm{K_\mathrm{ff}}^{-1} \bm{W_\mathrm{fu}} (\bm{V_\mathrm{uu}}^{-1} + {\bm{K_\mathrm{uu}}}^{-1})^{-1} \bm{V_\mathrm{uu}}^{-1} \bm{\bar{y}}$ \\  
    $\bm{\Sigma_p}$ & $\bm{K_\mathrm{**}} - \bm{K_\mathrm{*f}} \bm{K_\mathrm{ff}}^{-1} \bm{K_\mathrm{f*}} + \bm{K_\mathrm{*f}} \bm{K_\mathrm{ff}}^{-1} \bm{W_\mathrm{fu}} (\bm{V_\mathrm{uu}}^{-1} + {\bm{K_\mathrm{uu}}}^{-1})^{-1}  \bm{W_\mathrm{uf}} \bm{K_\mathrm{ff}}^{-1} \bm{K_\mathrm{f*}}$ \\ 
    $\bm{\mu_q}$ & $\bm{K_\mathrm{*u}} (\bm{K_\mathrm{uu}} + \bm{V_\mathrm{uu}})^{-1} \bm{\bar{y}}$ \\ 
    $\bm{\Sigma_q}$ & $\bm{K_\mathrm{**}} - \bm{K_\mathrm{*u}} (\bm{K_\mathrm{uu}} + \bm{V_\mathrm{uu}})^{-1} \bm{K_\mathrm{u*}}$  \\ 
    $\bm{V_\mathrm{uu}}$ & $\diag(n_1^{-1} v(\bar{y}_1), \cdots, n_m^{-1} v(\bar{y}_m))$ with $v : \mathbb{R} \rightarrow (0, \infty)$ \\   
    $p(\bm{u})$ & $\mathcal{N} (\bm{u}; \bm{0}, \bm{K_\mathrm{uu}})$ \\  
    $p(\bm{y} \mid \bm{W_\mathrm{fu}} \bm{u})$ & $p(\bm{y} \mid \bm{f})|_{\bm{f} = \bm{W_\mathrm{fu}} \bm{u}}$ \\ 
    $p(\bm{f} \mid \bm{u})$ & $\mathcal{N} (\bm{f}; \bm{K_\mathrm{fu}} {\bm{K_\mathrm{uu}}}^{-1} \bm{u}, \bm{K_\mathrm{ff}} - \bm{K_\mathrm{fu}} \bm{K_\mathrm{uu}}^{-1} \bm{K_\mathrm{uf}})$ \\     
    $p(\bm{f_\mathrm{*}} \mid \bm{W_\mathrm{fu}} \bm{u})$ & $p(\bm{f_\mathrm{*}} \mid \bm{f})|_{\bm{f} = \bm{W_\mathrm{fu}} \bm{u}}$ \\ 
    $p(\bm{f_\mathrm{*}} \mid \bm{y}, \bm{\omega})$ & $\exp(-\mathcal{E}) \int_{\mathbb{R}^m} p(\bm{f_\mathrm{*}} \mid \bm{W_\mathrm{fu}} \bm{u}) p(\bm{y} \mid \bm{W_\mathrm{fu}} \bm{u}) p(\bm{u}) d \bm{u}$ \\   
    $\mathcal{E}$ & $\log \int_{\mathbb{R}^m} p(\bm{y} \mid \bm{W_\mathrm{fu}} \bm{u}) p(\bm{u}) d \bm{u}$ \\  
    $\mathcal{Q}$ & $- \frac{m}{2} \log (2 \pi) - \frac{1}{2} \log \lvert \bm{K_\mathrm{uu}} + \bm{V_\mathrm{uu}} \rvert - \frac{1}{2} {\bm{\bar{y}}}^{\top} (\bm{K_\mathrm{uu}} + \bm{V_\mathrm{uu}})^{-1} \bm{\bar{y}}$  \\ 
    $R(\bm{u}, \delta)$ & $\{\bm{f} \in \mathbb{R}^n \mid \lvert\bm{f} - \bm{W_\mathrm{fu}} \bm{u}\rvert \leq \sqrt{n} \delta\}$ for any $\bm{u} \in \mathbb{R}^m$ and $\delta \in [0, \infty)$ \\   
    $\epsilon(\delta_1, \delta_2)$ & $F \bigl(m, (\delta_1 - \delta_2)^2 \lambda_1^{-1} \kappa m \bigr) + F \bigl(n, \delta_2^2 \lambda_2^{-1} n\bigr)$ \\ 
    $\kappa$ & $\inf_{\bm{u} \in \mathbb{R}^m; \lvert(\bm{W_\mathrm{fu}} - \bm{K_\mathrm{fu}} {\bm{K_\mathrm{uu}}}^{-1}) \bm{u}\rvert = 1} \frac{n}{m} \lvert\bm{u}\rvert^2$ \\ 
    $\lambda_1$ & the maximum eigenvalue of $\bm{K_\mathrm{uu}}$ \\  
    $\lambda_2$ & the maximum eigenvalue of $\bm{K_\mathrm{ff}} - \bm{K_\mathrm{fu}} \bm{K_\mathrm{uu}}^{-1} \bm{K_\mathrm{uf}}$ \\  
    $\eta$ & $\inf_{\xi \in [\xi_0, \infty)} (\sqrt{\xi \lambda_1 \kappa^{-1}} + \sqrt{\xi \lambda_2} + F(m, \xi m))$ with the larger $\xi_0$ such that $2 \sqrt{\pi \xi_0} \exp(- \xi_0) = 1$ \\  
    $\gamma$ & the maximum absolute value of the elements in $\bm{W_\mathrm{fu}} - \bm{K_\mathrm{fu}} {\bm{K_\mathrm{uu}}}^{-1}$ \\ 
    $F (a, b)$ & $(\int_{0}^{\infty} t^{\frac{a}{2} - 1} \exp(- t) dt)^{-1} \int_{\frac{b}{2}}^{\infty} t^{\frac{a}{2} - 1} \exp(- t) dt$ for any $a \in \mathbb{N}$ and $b \in [0, \infty)$ \\  
    \bottomrule 
  \end{tabular}
  \end{threeparttable}
  \label{symbol_description}
\end{table*}

\newpage 
\section{Input Approximation}

\subsection{Error Bounds for Covariance Functions}  
\label{example_covariance}
Let $\max\{|\bm{x} - \bm{z}|, |\bm{x^\prime} - \bm{z^\prime}|\} < \alpha$. 
For the Laplacian kernel, the following holds: 
\begin{align} 
  |\exp (- |\bm{x} - \bm{x^\prime}|) - \exp (- |\bm{z} - \bm{z^\prime}|)| 
  &\leq  1 - \exp (- ||\bm{z} - \bm{z^\prime}| - |\bm{x} - \bm{x^\prime}||)&  \nonumber  \\   
  &\leq  1 - \exp (- |\bm{z} - \bm{x} + \bm{x^\prime} - \bm{z^\prime}|)&  \nonumber  \\   
  &\leq  1 - \exp (- |\bm{z} - \bm{x}| - |\bm{x^\prime} - \bm{z^\prime}|)&  \nonumber  \\   
  &< 1 - \exp (- 2 \alpha).&  \nonumber 
\end{align} 
Similarly, the following holds: 
\begin{align} 
  |\exp (- |\bm{x} - \bm{x^\mathrm{*}}|) - \exp (- |\bm{z} - \bm{x^\mathrm{*}}|)|  
  &\leq  1 - \exp (- ||\bm{z} - \bm{x^\mathrm{*}}| - |\bm{x} - \bm{x^\mathrm{*}}||)&  \nonumber  \\   
  &\leq  1 - \exp (- |\bm{x} - \bm{z}|)&  \nonumber  \\   
  &< 1 - \exp (- \alpha).&  \nonumber
\end{align} 
For the Gaussian kernel, the following holds: 
\begin{align} 
  \Bigl|\exp \Bigl(- \frac{1}{2} |\bm{x} - \bm{x^\prime}|^2 \Bigr) - \exp \Bigl(- \frac{1}{2} |\bm{z} - \bm{z^\prime}|^2 \Bigr)\Bigr|
  &\leq 1 - \exp \Bigl(- \frac{1}{2} \bigl(||\bm{z} - \bm{z^\prime}|^2 - |\bm{x} - \bm{x^\prime}|^2|\bigr) \Bigr)& \nonumber  \\ 
  &\leq 1 - \exp \Bigl(- \frac{1}{2} \bigl((|\bm{z} - \bm{z^\prime}| + |\bm{x} - \bm{z}| + |\bm{x^\prime} - \bm{z^\prime}|)^2 - |\bm{z} - \bm{z^\prime}|^2\bigr) \Bigr)& \nonumber  \\ 
  &< 1 - \exp \Bigl(- \frac{1}{2} \bigl((|\bm{z} - \bm{z^\prime}| + 2\alpha )^2 - |\bm{z} - \bm{z^\prime}|^2\bigr) \Bigr)& \nonumber  \\ 
  &= 1 - \exp \bigl(- 2 \alpha (|\bm{z} - \bm{z^\prime}| + \alpha) \bigr).&  \nonumber    
\end{align} 
Similarly, the following holds: 
\begin{align} 
  \Bigl|\exp \Bigl(- \frac{1}{2} |\bm{x} - \bm{x^\mathrm{*}}|^2 \Bigr) - \exp \Bigl(- \frac{1}{2} |\bm{z} - \bm{x^\mathrm{*}}|^2 \Bigr)\Bigr| 
  &\leq 1 - \exp \Bigl(- \frac{1}{2} \bigl(||\bm{z} - \bm{x^\mathrm{*}}|^2 - |\bm{x} - \bm{x^\mathrm{*}}|^2|\bigr) \Bigr)& \nonumber  \\ 
  &\leq 1 - \exp \Bigl(- \frac{1}{2} \bigl((|\bm{x} - \bm{x^\mathrm{*}}| + |\bm{x} - \bm{z}|)^2 - |\bm{x} - \bm{x^\mathrm{*}}|^2\bigr) \Bigr)& \nonumber  \\ 
  &< 1 - \exp \Bigl(- \frac{1}{2} \bigl((|\bm{x} - \bm{x^\mathrm{*}}| + \alpha)^2 - |\bm{x} - \bm{x^\mathrm{*}}|^2\bigr) \Bigr)& \nonumber  \\ 
  &= 1 - \exp \Bigl(- \frac{1}{2} \alpha (2 |\bm{x} - \bm{x^\mathrm{*}}| + \alpha) \Bigr)& \nonumber \\     
  &\leq 1 - \exp \Bigl(- \frac{1}{2} \alpha (2 |\bm{z} - \bm{x^\mathrm{*}}| + 2 |\bm{x} - \bm{z}| + \alpha) \Bigr)& \nonumber \\ 
  &< 1 - \exp \Bigl(- \alpha \Bigl(|\bm{z} - \bm{x^\mathrm{*}}| + \frac{3}{2} \alpha\Bigr) \Bigr).&  \nonumber    
\end{align}

\subsection{Proof of \Cref{kernel1}}
\label{kernel_proof}
From the definition of continuous function, for any $\beta \in (0, \infty)$, there exists $\alpha \in (0, \infty)$ such that 
\begin{align} 
  \max\{|\bm{x} - \bm{z}|, |\bm{x^\prime} - \bm{z^\prime}|\} < \alpha \Rightarrow 
  \max\{|k(\bm{x}, \bm{x^\prime}) - k(\bm{z}, \bm{x^\prime})|, |k(\bm{z}, \bm{x^\prime}) - k(\bm{z}, \bm{z^\prime})|\} < \frac{\beta}{2}.   \nonumber  
\end{align} 
Additionally, the following holds: 
\begin{align} 
  |k(\bm{x}, \bm{x^\prime}) - k(\bm{z}, \bm{z^\prime})| \leq |k(\bm{x}, \bm{x^\prime}) - k(\bm{z}, \bm{x^\prime})| + |k(\bm{z}, \bm{x^\prime}) - k(\bm{z}, \bm{z^\prime})|.  \nonumber  
\end{align} 
Therefore, for any $\beta \in (0, \infty)$ and $\bm{x}, \bm{x^\prime} \in \mathcal{X}$, there exists $\alpha \in (0, \infty)$ such that 
$\max\{|\bm{x} - \bm{z}|, |\bm{x^\prime} - \bm{z^\prime}|\} < \alpha \Rightarrow |k(\bm{x}, \bm{x^\prime}) - k(\bm{z}, \bm{z^\prime})| < \beta$.  
From the definition of continuous function, for any $\beta \in (0, \infty)$ and $\bm{x}, \bm{x^\prime} \in \mathcal{X}$, there exists $\alpha \in (0, \infty)$ such that 
$|\bm{x} - \bm{z}| < \alpha \Rightarrow |k(\bm{x}, \bm{x^\mathrm{*}}) - k(\bm{z}, \bm{x^\mathrm{*}})| < \beta$.

\subsection{Proof of \Cref{kernel2}}
\label{kernel2_proof}
The following holds: 
\begin{align} 
  \bm{K_\mathrm{ff}} - \bm{K_\mathrm{fu}} \bm{K_\mathrm{uu}}^{-1} \bm{K_\mathrm{uf}} 
  &= (\bm{W_\mathrm{fu}} \bm{K_\mathrm{uu}} \bm{W_\mathrm{uf}} + O(\beta)) - (\bm{W_\mathrm{fu}} \bm{K_\mathrm{uu}} + (\bm{K_\mathrm{fu}} - \bm{W_\mathrm{fu}} \bm{K_\mathrm{uu}})) \bm{K_\mathrm{uu}}^{-1} (\bm{K_\mathrm{uu}} \bm{W_\mathrm{uf}} + (\bm{K_\mathrm{uf}} - \bm{K_\mathrm{uu}} \bm{W_\mathrm{uf}}))&   \nonumber \\ 
  &= - \bm{W_\mathrm{fu}} (\bm{K_\mathrm{uf}} - \bm{K_\mathrm{uu}} \bm{W_\mathrm{uf}}) - (\bm{K_\mathrm{fu}} - \bm{W_\mathrm{fu}} \bm{K_\mathrm{uu}}) \bm{W_\mathrm{uf}} + (\bm{W_\mathrm{fu}} - \bm{K_\mathrm{fu}} \bm{K_\mathrm{uu}}^{-1}) (\bm{K_\mathrm{uf}} - \bm{K_\mathrm{uu}} \bm{W_\mathrm{uf}}) + O(\beta)&   \nonumber \\ 
  &= O(\beta + m \beta \gamma).&  \nonumber  
\end{align}

\subsection{Proof of \Cref{lem_bound}}  
\label{proof_lem_bound}
For $\delta \in [0, \infty)$, define 
\begin{align} 
  R_1(\delta) \equiv \{\bm{u} \in \mathbb{R}^m \mid \lvert(\bm{W_\mathrm{fu}} - \bm{K_\mathrm{fu}} {\bm{K_\mathrm{uu}}}^{-1}) \bm{u}\rvert \leq \sqrt{n}\delta\}, \nonumber \\ 
  R_2(\delta) \equiv \{\bm{u} \in \mathbb{R}^m \mid \lvert(\bm{W_\mathrm{fu}} - \bm{K_\mathrm{fu}} {\bm{K_\mathrm{uu}}}^{-1}) \bm{P_\mathrm{uu}}^\top \bm{u}\rvert \leq \sqrt{n}\delta\},  \nonumber \\  
  R_3(\bm{u}, \delta) \equiv \{\bm{f} \in \mathbb{R}^n \mid \lvert\bm{f} - \bm{K_\mathrm{fu}} {\bm{K_\mathrm{uu}}}^{-1} \bm{u}\rvert \leq \sqrt{n} \delta\}. \nonumber  
\end{align} 
The following holds:  
\begin{align}
&\int_{\mathbb{R}^m} \int_{\mathbb{R}^n \setminus R(\bm{u}, \delta_1)} p(\bm{f} \mid \bm{u}) p(\bm{u}) d\bm{f} d\bm{u}&  \nonumber  \\  
&= \int_{\mathbb{R}^m \setminus R_1(\delta_1 - \delta_2)} \int_{\mathbb{R}^n \setminus R(\bm{u}, \delta_1)} p(\bm{f} \mid \bm{u}) p(\bm{u}) d\bm{f} d\bm{u} 
 + \int_{R_1(\delta_1 - \delta_2)} \int_{\mathbb{R}^n \setminus R(\bm{u}, \delta_1)} p(\bm{f} \mid \bm{u}) p(\bm{u}) d\bm{f} d\bm{u}&  \nonumber  \\
&\leq \int_{\mathbb{R}^m \setminus R_1(\delta_1 - \delta_2)} p(\bm{u}) d\bm{u} 
 + \int_{\mathbb{R}^m} \int_{\mathbb{R}^n \setminus R_3(\bm{u}, \delta_2)} p(\bm{f} \mid \bm{u}) p(\bm{u}) d\bm{f} d\bm{u}&  \nonumber  \\
&= \int_{\mathbb{R}^m \setminus R_1(\delta_1 - \delta_2)} p(\bm{u}) d\bm{u}  
 + \int_{\mathbb{R}^n \setminus R_3(\bm{0}, \delta_2)} p(\bm{f} \mid \bm{u})\bigl|_{\bm{u} = \bm{0}} d\bm{f}& \nonumber  \\   
&\leq \int_{\mathbb{R}^m \setminus R_1(\delta_1 - \delta_2)} (2 \pi)^{- \frac{m}{2}} |\bm{\Lambda_\mathrm{uu}}|^{- \frac{1}{2}} \exp \Bigl(- \frac{1}{2} (\bm{P_\mathrm{uu}} \bm{u})^\top \bm{\Lambda_\mathrm{uu}}^{-1} (\bm{P_\mathrm{uu}} \bm{u})\Bigl) d\bm{u}& \nonumber  \\     
&~~~~+ \int_{\mathbb{R}^n \setminus R_3(\bm{0}, \delta_2)} (2 \pi)^{- \frac{n}{2}} |\bm{\Lambda_\mathrm{ff}}|^{- \frac{1}{2}} \exp \Bigl(- \frac{1}{2} (\bm{P_\mathrm{ff}} \bm{f})^\top \bm{\Lambda_\mathrm{ff}}^{-1} (\bm{P_\mathrm{ff}} \bm{f})\Bigl) d\bm{f}& \nonumber \\  
&= \int_{\mathbb{R}^m \setminus R_2(\delta_1 - \delta_2)} (2 \pi)^{- \frac{m}{2}} |\bm{\Lambda_\mathrm{uu}}|^{- \frac{1}{2}} \exp \Bigl(- \frac{1}{2} \bm{u}^\top \bm{\Lambda_\mathrm{uu}}^{-1} \bm{u}\Bigl) d\bm{u}& \nonumber  \\  
&~~~~+ \int_{\mathbb{R}^n \setminus R_3(\bm{0}, \delta_2)} (2 \pi)^{- \frac{n}{2}} |\bm{\Lambda_\mathrm{ff}}|^{- \frac{1}{2}} \exp \Bigl(- \frac{1}{2} \bm{f}^\top \bm{\Lambda_\mathrm{ff}}^{-1} \bm{f} \Bigl) d\bm{f}& \nonumber  \\ 
&\leq \int_{\mathbb{R}^m \setminus R_1\bigl(\frac{\delta_1 - \delta_2}{\sqrt{\lambda_1}} \sqrt{\frac{m}{n} \kappa} \bigr)} (2 \pi)^{- \frac{m}{2}} \exp \Bigl(- \frac{1}{2} \bm{u}^\top \bm{u}\Bigl) d\bm{u}& \nonumber  \\ 
&~~~~+ \int_{\mathbb{R}^n \setminus R_3\bigl(\bm{0}, \frac{\delta_2}{\sqrt{\lambda_2}}\bigr)} (2 \pi)^{- \frac{n}{2}} \exp \Bigl(- \frac{1}{2} \bm{f}^\top \bm{f} \Bigl) d\bm{f}& \nonumber  \\ 
&= \epsilon(\delta_1, \delta_2), \nonumber 
\end{align} 
where $\bm{P_\mathrm{uu}}$ and $\bm{\Lambda_\mathrm{uu}}$ are the orthogonal matrix and eigenvalue matrix such that $\bm{K_\mathrm{uu}} = \bm{P_\mathrm{uu}}^\top \bm{\Lambda_\mathrm{uu}} \bm{P_\mathrm{uu}}$, respectively, 
and $\bm{P_\mathrm{ff}}$ and $\bm{\Lambda_\mathrm{ff}}$ are the orthogonal matrix and eigenvalue matrix such that $\bm{K_\mathrm{ff}} - \bm{K_\mathrm{fu}} \bm{K_\mathrm{uu}}^{-1} \bm{K_\mathrm{uf}} = \bm{P_\mathrm{ff}}^\top \bm{\Lambda_\mathrm{ff}} \bm{P_\mathrm{ff}}$, respectively.

\subsection{Proof of \Cref{convert}}
\label{convert_proof}
Consider the index $i$ for which $\lvert[(\bm{W_\mathrm{fu}} - \bm{K_\mathrm{fu}} {\bm{K_\mathrm{uu}}}^{-1}) \bm{u}]_i\rvert$ is maximum in the equation of $\kappa$. 
The following holds: 
\begin{align} 
  \kappa 
  &\geq \inf_{\bm{u} \in \mathbb{R}^m; \lvert [(\bm{W_\mathrm{fu}} - \bm{K_\mathrm{fu}} {\bm{K_\mathrm{uu}}}^{-1}) \bm{u}]_i \rvert \geq \frac{1}{\sqrt{n}}} \frac{n}{m} \lvert\bm{u}\rvert^2&  \nonumber \\ 
  &\geq \inf_{\bm{u} \in \mathbb{R}^m; \lvert \sum_{1 \leq j \leq m} [\bm{W_\mathrm{fu}} - \bm{K_\mathrm{fu}} {\bm{K_\mathrm{uu}}}^{-1}]_{ij} [\bm{u}]_j \rvert \geq \frac{1}{\sqrt{n}}} \frac{n}{m} \lvert\bm{u}\rvert^2&  \nonumber \\ 
  &= \frac{1}{m \max_{1 \leq j \leq m} [\bm{W_\mathrm{fu}} - \bm{K_\mathrm{fu}} {\bm{K_\mathrm{uu}}}^{-1}]_{ij}^2}&  \nonumber \\ 
  &\geq \frac{1}{m \gamma^2}.&  \nonumber 
\end{align} 
Consequently, we have 
\begin{align} 
  \kappa^{-1} = O(m \gamma^2). \nonumber 
\end{align} 
Here the following property is well known. 
\begin{property}
  The square of the Frobenius norm of a real symmetric matrix equals the sum of the squares of its eigenvalues. 
\end{property}
Using this property, we have  
\begin{align} 
  \lambda_2^2 \leq \sum_{1 \leq i \leq n, 1 \leq j \leq n} [\bm{K_\mathrm{ff}} - \bm{K_\mathrm{fu}} \bm{K_\mathrm{uu}}^{-1} \bm{K_\mathrm{uf}}]_{ij}^2. \nonumber 
\end{align} 
Therefore, from \cref{kernel2}, the following holds: 
\begin{align} 
  \lambda_2 = O(\beta + m \beta \gamma). \nonumber 
\end{align}

\subsection{Proof of \Cref{monotonically}}
\label{proof_monotonically} 
The following property is often used to analyze the gamma function. 
\begin{property}
  \label{degamma} 
  Let $\psi$ denote the derivative of the logarithm of the gamma function. 
  The following properties hold: 
  \begin{itemize}
    \item $\psi$ is monotonically increasing over the positive real numbers.  
    \item The derivative of $\psi$ is monotonically decreasing over the positive real numbers. 
    \item $\psi(t + 1) - \psi(t) = t^{-1}$ holds for any positive real number $t$. 
  \end{itemize} 
\end{property} 
Let $g_1$ and $g_2$ be defined as  
\begin{align}
  g_1(t, m) \equiv \frac{m (m t)^{\frac{m}{2} - 1} \exp(- m t)}{\Gamma(\frac{m}{2})}, 
  ~~g_2(m) \equiv \frac{\Gamma(\frac{m}{2})}{\Gamma(\frac{m + 1}{2})} \frac{(m + 1)^{\frac{m + 1}{2}}}{m^{\frac{m}{2}}}, \nonumber 
\end{align} 
where $\Gamma$ is the gamma function. 
To remove $m$ from the lower limit of integration, we express $F(m, \xi m)$ as follows:   
\begin{align}
  F(m, \xi m) = \frac{\int_{\frac{\xi m}{2}}^{\infty} t^{\frac{m}{2} - 1} \exp(- t) dt}{\Gamma(\frac{m}{2})} = \int_{\frac{\xi}{2}}^{\infty} g_1(t, m) dt. \nonumber 
\end{align} 
For the integrand, we have   
\begin{align}
  \frac{g_1(t, m + 1)}{g_1(t, m)} = g_2(m) \sqrt{t} \exp(- t).   \nonumber 
\end{align} 
Using \cref{degamma}, the following holds under the condition that $m$ is a positive real number: 
\begin{align} 
  \frac{d \log g_2(m)}{d m} 
  &= \frac{d}{d m} \Bigl(\log\Gamma\Bigl(\frac{m}{2}\Bigr) - \log\Gamma\Bigl(\frac{m + 1}{2}\Bigr) + \frac{m + 1}{2} \log (m + 1) - \frac{m}{2} \log m\Bigr)&   \nonumber \\ 
  &= \frac{1}{2} \Bigl(\psi\Bigl(\frac{m}{2}\Bigr) - \psi\Bigl(\frac{m + 1}{2}\Bigr) + \log \Bigl(1 + \frac{1}{m}\Bigr)\Bigr)&   \nonumber \\ 
  &< \frac{1}{2} \Bigl(\frac{1}{2} \psi\Bigl(\frac{m}{2}\Bigr) - \frac{1}{2} \psi\Bigl(\frac{m}{2} + 1\Bigr) + \frac{1}{m}\Bigr)&   \nonumber \\ 
  &= \frac{1}{2} \Bigl(- \frac{1}{m} + \frac{1}{m}\Bigr)&   \nonumber \\ 
  &= 0.&   \nonumber 
\end{align} 
Consequently, for any $t \geq \xi_0$ and $m \in \{2, 3, \cdots \}$, we have 
\begin{align}
  \frac{g_1(t, m + 1)}{g_1(t, m)} = g_2(m) \sqrt{t} \exp(- t) < g_2(1) \sqrt{t} \exp(- t) = 2 \sqrt{\pi t} \exp(- t) \leq 1.   \nonumber 
\end{align} 
Additonally, $g_1(t, 1) > g_1(t, 2)$ holds.  
Therefore, for any $\xi \geq \xi_0$ and $m \in \{1, 2, \cdots \}$, the following holds: 
\begin{align}
  \frac{F(m + 1, \xi (m + 1))}{F(m, \xi m)} = \frac{\int_{\frac{\xi}{2}}^{\infty} g_1(t, m + 1) dt}{\int_{\frac{\xi}{2}}^{\infty} g_1(t, m) dt} < \frac{\int_{\frac{\xi}{2}}^{\infty} g_1(t, m) dt}{\int_{\frac{\xi}{2}}^{\infty} g_1(t, m) dt} = 1.  \nonumber 
\end{align} 
From this result, $F(m, \xi m)$ is monotonically decreasing with respect to $m$.

\subsection{Proof of \Cref{consistency}}
\label{bound_AP}
The following lemma provides an upper bound for $|\mathcal{L} - \mathcal{E}|$.   
\begin{lemma}
\label{bound1}
Let $a_1$ and $a_2$ be defined as  
\begin{align}
  a_1 \equiv \exp(- \mathcal{L}) \sup_{\bm{f} \in \mathbb{R}^n} p(\bm{y} \mid \bm{f}), 
  ~~a_2 \equiv \exp(- \mathcal{L}) \sup_{\bm{f} \in \mathbb{R}^n} \lvert\nabla p(\bm{y} \mid \bm{f}) \rvert.   \nonumber 
\end{align} 
Then we have 
\begin{align}
  |\mathcal{L} - \mathcal{E}| \leq a_1 \epsilon(\delta_1, \delta_2) + a_2 \delta_1.  \nonumber 
\end{align} 
\end{lemma} 
\begin{proof}
The following holds: 
\begin{align}
  |p(\bm{y} \mid \bm{f}) - p(\bm{y} \mid \bm{W_\mathrm{fu}} \bm{u})| 
  &\leq \sup_{\bm{f^\prime} \in R(\bm{u}, \delta_1)} \sup_{\bm{v} \in \mathbb{R}^n; |\bm{v}| = 1} | \nabla_{\bm{v}} p(\bm{y} \mid \bm{f^\prime}) | \delta_1& \nonumber \\ 
  &\leq \sup_{\bm{f^\prime} \in \mathbb{R}^n} \sup_{\bm{v} \in \mathbb{R}^n; |\bm{v}| = 1} | \nabla_{\bm{v}} p(\bm{y} \mid \bm{f^\prime}) | \delta_1& \nonumber \\ 
  &= \exp(\mathcal{L}) a_2 \delta_1&  \nonumber 
\end{align} 
for any $\bm{f} \in R(\bm{u}, \delta_1)$ and $\bm{u} \in \mathbb{R}^m$. 
Here the marginal likelihood $\int_{\mathbb{R}^n} p(\bm{y} \mid \bm{f}) p(\bm{f}) d \bm{f}$ is divided as 
\begin{align}
  \int_{\mathbb{R}^m} \int_{R(\bm{u}, \delta_1)} p(\bm{y} \mid \bm{f}) p(\bm{f} \mid \bm{u}) p(\bm{u}) d\bm{f} d\bm{u}  
  + \int_{\mathbb{R}^m} \int_{\mathbb{R}^n \setminus R(\bm{u}, \delta_1)} p(\bm{y} \mid \bm{f}) p(\bm{f} \mid \bm{u}) p(\bm{u}) d\bm{f} d\bm{u}.  \nonumber
\end{align} 
The upper bound of the first term is 
\begin{align}
  \int_{\mathbb{R}^m} \int_{R(\bm{u}, \delta_1)} p(\bm{y} \mid \bm{f}) p(\bm{f} \mid \bm{u}) p(\bm{u}) d\bm{f} d\bm{u}   
  &\leq \int_{\mathbb{R}^m} \int_{R(\bm{u}, \delta_1)} p(\bm{y} \mid \bm{W_\mathrm{fu}} \bm{u}) p(\bm{f} \mid \bm{u}) p(\bm{u}) d\bm{f} d\bm{u}&  \nonumber \\  
  &~~~~+ \int_{\mathbb{R}^m} \int_{R(\bm{u}, \delta_1)}  |p(\bm{y} \mid \bm{f}) - p(\bm{y} \mid \bm{W_\mathrm{fu}} \bm{u})| p(\bm{f} \mid \bm{u}) p(\bm{u}) d\bm{f} d\bm{u}&  \nonumber \\ 
  &\leq \int_{\mathbb{R}^m} \int_{R(\bm{u}, \delta_1)} (p(\bm{y} \mid \bm{W_\mathrm{fu}} \bm{u}) + \exp(\mathcal{L}) a_2 \delta_1) p(\bm{f} \mid \bm{u}) p(\bm{u}) d\bm{f} d\bm{u}&  \nonumber \\ 
  &\leq \int_{\mathbb{R}^m} p(\bm{y} \mid \bm{W_\mathrm{fu}} \bm{u}) p(\bm{u}) d \bm{u} + \exp(\mathcal{L}) a_2 \delta_1.& \nonumber
\end{align} 
Using \cref{lem_bound}, the upper bound of the second term is 
\begin{align}
  \int_{\mathbb{R}^m} \int_{\mathbb{R}^n \setminus R(\bm{u}, \delta_1)} p(\bm{y} \mid \bm{f}) p(\bm{f} \mid \bm{u}) p(\bm{u}) d\bm{f} d\bm{u} \leq \exp(\mathcal{L}) a_1 \epsilon(\delta_1, \delta_2).   \nonumber  
\end{align} 
From \cref{lem_bound}, the lower bound of the first term is 
\begin{align}
  \int_{\mathbb{R}^m} \int_{R(\bm{u}, \delta_1)} p(\bm{y} \mid \bm{f}) p(\bm{f} \mid \bm{u}) p(\bm{u}) d\bm{f} d\bm{u}  
  &\geq \int_{\mathbb{R}^m} \int_{R(\bm{u}, \delta_1)} p(\bm{y} \mid \bm{W_\mathrm{fu}} \bm{u}) p(\bm{f} \mid \bm{u}) p(\bm{u}) d\bm{f} d\bm{u}&  \nonumber \\  
  &~~~~- \int_{\mathbb{R}^m} \int_{R(\bm{u}, \delta_1)}  |p(\bm{y} \mid \bm{f}) - p(\bm{y} \mid \bm{W_\mathrm{fu}} \bm{u})| p(\bm{f} \mid \bm{u}) p(\bm{u}) d\bm{f} d\bm{u}&  \nonumber \\ 
  &\geq \int_{\mathbb{R}^m} \int_{R(\bm{u}, \delta_1)} (p(\bm{y} \mid \bm{W_\mathrm{fu}} \bm{u}) - \exp(\mathcal{L}) a_2 \delta_1) p(\bm{f} \mid \bm{u}) p(\bm{u}) d\bm{f} d\bm{u}&  \nonumber \\ 
  &\geq \int_{\mathbb{R}^m} p(\bm{y} \mid \bm{W_\mathrm{fu}} \bm{u}) p(\bm{u}) d \bm{u} - \exp(\mathcal{L}) a_1 \epsilon(\delta_1, \delta_2) - \exp(\mathcal{L}) a_2 \delta_1.&     \nonumber  
\end{align} 
The lower bound of the second term is 
\begin{align}
  \int_{\mathbb{R}^m} \int_{\mathbb{R}^n \setminus R(\bm{u}, \delta_1)} p(\bm{y} \mid \bm{f}) p(\bm{f} \mid \bm{u}) p(\bm{u}) d\bm{f} d\bm{u} \geq 0.    \nonumber  
\end{align} 
From the above results, the following holds: 
\begin{align}
\Bigl|\int_{\mathbb{R}^n} p(\bm{y} \mid \bm{f}) p(\bm{f}) d \bm{f} - \int_{\mathbb{R}^m} p(\bm{y} \mid \bm{W_\mathrm{fu}} \bm{u}) p(\bm{u}) d \bm{u} \Bigr| \leq \exp(\mathcal{L}) (a_1 \epsilon(\delta_1, \delta_2) + a_2 \delta_1).  \nonumber 
\end{align} 
Therefore, the following holds: 
\begin{align}
  |\mathcal{L} - \mathcal{E}| \leq \log(1 + a_1 \epsilon(\delta_1, \delta_2) + a_2 \delta_1) \leq a_1 \epsilon(\delta_1, \delta_2) + a_2 \delta_1.  \nonumber 
\end{align} 
\end{proof}
Additionally, the following lemma provides an upper bound for $\lVert\mathbb{E}_{p(\bm{f_\mathrm{*}} \mid \bm{y})} [\bm{f_\mathrm{*}}] - \mathbb{E}_{p(\bm{f_\mathrm{*}} \mid \bm{y}, \bm{\omega})} [\bm{f_\mathrm{*}}]\rVert$.   
\begin{lemma}
\label{bound2} 
Let $B_1$ and $B_2$ be defined as  
\begin{align} 
  B_1 (\bm{f_\mathrm{*}}) \equiv \exp(- \mathcal{L}) \sup_{\bm{f} \in \mathbb{R}^n} p(\bm{f_\mathrm{*}} \mid \bm{f}) p(\bm{y} \mid \bm{f}), 
  ~~B_2 (\bm{f_\mathrm{*}}) \equiv \exp(- \mathcal{L}) \sup_{\bm{f} \in \mathbb{R}^n} \lvert\nabla (p(\bm{f_\mathrm{*}} \mid \bm{f}) p(\bm{y} \mid \bm{f}))\rvert.   \nonumber 
\end{align} 
Then we have 
\begin{align}
  \bigl\lVert\mathbb{E}_{p(\bm{f_\mathrm{*}} \mid \bm{y})} [\bm{f_\mathrm{*}}] - \mathbb{E}_{p(\bm{f_\mathrm{*}} \mid \bm{y}, \bm{\omega})} [\bm{f_\mathrm{*}}]\bigr\rVert
  \leq (1 + a_1) b_1 \epsilon(\delta_1, \delta_2) + (a_2 b_1 + b_2) \delta_1,    \nonumber  
\end{align} 
where $b_1 \equiv \int_{\mathbb{R}^{n_\mathrm{*}}} \lVert\bm{f_\mathrm{*}}\rVert B_1(\bm{f_\mathrm{*}}) d\bm{f_\mathrm{*}}$ and $b_2 \equiv \int_{\mathbb{R}^{n_\mathrm{*}}} \lVert\bm{f_\mathrm{*}}\rVert B_2(\bm{f_\mathrm{*}}) d\bm{f_\mathrm{*}}$. 
\end{lemma}
\begin{proof}
The following holds: 
\begin{align}
  |p(\bm{f_\mathrm{*}} \mid \bm{f}) p(\bm{y} \mid \bm{f}) - p(\bm{f_\mathrm{*}} \mid \bm{W_\mathrm{fu}} \bm{u}) p(\bm{y} \mid \bm{W_\mathrm{fu}} \bm{u})|     
  &\leq \sup_{\bm{f^\prime} \in R(\bm{u}, \delta_1)} \sup_{\bm{v} \in \mathbb{R}^n; |\bm{v}| = 1}| \nabla_{\bm{v}} (p(\bm{f_\mathrm{*}} \mid \bm{f^\prime}) p(\bm{y} \mid \bm{f^\prime})) | \delta_1& \nonumber \\ 
  &\leq \sup_{\bm{f^\prime} \in \mathbb{R}^n} \sup_{\bm{v} \in \mathbb{R}^n; |\bm{v}| = 1}  |\nabla_{\bm{v}} (p(\bm{f_\mathrm{*}} \mid \bm{f^\prime}) p(\bm{y} \mid \bm{f^\prime})) | \delta_1& \nonumber \\ 
  &= \exp(\mathcal{L}) B_2 (\bm{f_\mathrm{*}}) \delta_1&  \nonumber 
\end{align} 
for any $\bm{f} \in R(\bm{u}, \delta_1)$, $\bm{u} \in \mathbb{R}^m$, and $\bm{f_\mathrm{*}} \in \mathbb{R}^{n_\mathrm{*}}$. 
Here the posterior distribution $p(\bm{f_\mathrm{*}} \mid \bm{y})$ is divided as 
\begin{align}
\exp(- \mathcal{L}) \Bigl(\int_{\mathbb{R}^m} \int_{R(\bm{u}, \delta_1)} p(\bm{f_\mathrm{*}} \mid \bm{f}) p(\bm{y} \mid \bm{f}) p(\bm{f} \mid \bm{u}) p(\bm{u}) d\bm{f} d\bm{u} 
+ \int_{\mathbb{R}^m} \int_{\mathbb{R}^n \setminus R(\bm{u}, \delta_1)} p(\bm{f_\mathrm{*}} \mid \bm{f}) p(\bm{y} \mid \bm{f}) p(\bm{f} \mid \bm{u}) p(\bm{u}) d\bm{f} d\bm{u}\Bigr). \nonumber 
\end{align} 
The upper bound of the first term is 
\begin{align}
&\exp(- \mathcal{L}) \int_{\mathbb{R}^m} \int_{R(\bm{u}, \delta_1)} p(\bm{f_\mathrm{*}} \mid \bm{f}) p(\bm{y} \mid \bm{f}) p(\bm{f} \mid \bm{u}) p(\bm{u}) d\bm{f} d\bm{u}&  \nonumber \\  
&\leq \exp(- \mathcal{L}) \int_{\mathbb{R}^m} \int_{R(\bm{u}, \delta_1)} p(\bm{f_\mathrm{*}} \mid \bm{W_\mathrm{fu}} \bm{u}) p(\bm{y} \mid \bm{W_\mathrm{fu}} \bm{u}) p(\bm{f} \mid \bm{u}) p(\bm{u}) d\bm{f} d\bm{u}&   \nonumber \\ 
&~~~~+ \exp(- \mathcal{L}) \int_{\mathbb{R}^m} \int_{R(\bm{u}, \delta_1)} |p(\bm{f_\mathrm{*}} \mid \bm{f}) p(\bm{y} \mid \bm{f}) - p(\bm{f_\mathrm{*}} \mid \bm{W_\mathrm{fu}} \bm{u}) p(\bm{y} \mid \bm{W_\mathrm{fu}} \bm{u})| p(\bm{f} \mid \bm{u}) p(\bm{u}) d\bm{f} d\bm{u}&   \nonumber \\ 
&\leq \exp(- \mathcal{L}) \int_{\mathbb{R}^m} \int_{R(\bm{u}, \delta_1)} (p(\bm{f_\mathrm{*}} \mid \bm{W_\mathrm{fu}} \bm{u}) p(\bm{y} \mid \bm{W_\mathrm{fu}} \bm{u}) + \exp(\mathcal{L}) B_2(\bm{f_\mathrm{*}}) \delta_1) p(\bm{f} \mid \bm{u}) p(\bm{u}) d\bm{f} d\bm{u}&   \nonumber \\ 
&\leq \exp(- \mathcal{L} + \mathcal{E}) p(\bm{f_\mathrm{*}} \mid \bm{y}, \bm{\omega}) + B_2(\bm{f_\mathrm{*}}) \delta_1&  \nonumber \\    
&\leq p(\bm{f_\mathrm{*}} \mid \bm{y}, \bm{\omega}) + B_1(\bm{f_\mathrm{*}}) (\exp(|\mathcal{L} - \mathcal{E}|) - 1) + B_2(\bm{f_\mathrm{*}}) \delta_1.&  \nonumber   
\end{align} 
Using \cref{lem_bound}, the upper bound of the second term is 
\begin{align}
  \exp(- \mathcal{L}) \int_{\mathbb{R}^m} \int_{\mathbb{R}^n \setminus R(\bm{u}, \delta_1)} p(\bm{f_\mathrm{*}} \mid \bm{f}) p(\bm{y} \mid \bm{f}) p(\bm{f} \mid \bm{u}) p(\bm{u}) d\bm{f} d\bm{u} \leq B_1(\bm{f_\mathrm{*}}) \epsilon(\delta_1, \delta_2). \nonumber  
\end{align} 
From \cref{lem_bound}, the lower bound of the first term is 
\begin{align}
&\exp(- \mathcal{L}) \int_{\mathbb{R}^m} \int_{R(\bm{u}, \delta_1)} p(\bm{f_\mathrm{*}} \mid \bm{f}) p(\bm{y} \mid \bm{f}) p(\bm{f} \mid \bm{u}) p(\bm{u}) d\bm{f} d\bm{u}&  \nonumber \\  
&\geq \exp(- \mathcal{L}) \int_{\mathbb{R}^m} \int_{R(\bm{u}, \delta_1)} p(\bm{f_\mathrm{*}} \mid \bm{W_\mathrm{fu}} \bm{u}) p(\bm{y} \mid \bm{W_\mathrm{fu}} \bm{u}) p(\bm{f} \mid \bm{u}) p(\bm{u}) d\bm{f} d\bm{u}&   \nonumber \\ 
&~~~~- \exp(- \mathcal{L}) \int_{\mathbb{R}^m} \int_{R(\bm{u}, \delta_1)} |p(\bm{f_\mathrm{*}} \mid \bm{f}) p(\bm{y} \mid \bm{f}) - p(\bm{f_\mathrm{*}} \mid \bm{W_\mathrm{fu}} \bm{u}) p(\bm{y} \mid \bm{W_\mathrm{fu}} \bm{u})| p(\bm{f} \mid \bm{u}) p(\bm{u}) d\bm{f} d\bm{u}&   \nonumber \\ 
&\geq \exp(- \mathcal{L}) \int_{\mathbb{R}^m} \int_{R(\bm{u}, \delta_1)} p(\bm{f_\mathrm{*}} \mid \bm{W_\mathrm{fu}} \bm{u}) p(\bm{y} \mid \bm{W_\mathrm{fu}} \bm{u}) p(\bm{f} \mid \bm{u}) p(\bm{u}) d\bm{f} d\bm{u}&   \nonumber \\  
&~~~~- \int_{\mathbb{R}^m} \int_{R(\bm{u}, \delta_1)} B_2(\bm{f_\mathrm{*}}) \delta_1 p(\bm{f} \mid \bm{u}) p(\bm{u}) d\bm{f} d\bm{u}&   \nonumber \\ 
&\geq \exp(- \mathcal{L} + \mathcal{E}) p(\bm{f_\mathrm{*}} \mid \bm{y}, \bm{\omega}) - B_1(\bm{f_\mathrm{*}}) \epsilon(\delta_1, \delta_2) - B_2(\bm{f_\mathrm{*}}) \delta_1& \nonumber \\    
&\geq p(\bm{f_\mathrm{*}} \mid \bm{y}, \bm{\omega}) - B_1(\bm{f_\mathrm{*}}) (\exp(|\mathcal{L} - \mathcal{E}|) - 1) - B_1(\bm{f_\mathrm{*}}) \epsilon(\delta_1, \delta_2) - B_2(\bm{f_\mathrm{*}}) \delta_1.& \nonumber   
\end{align} 
The lower bound of the second term is 
\begin{align}
  \exp(- \mathcal{L}) \int_{\mathbb{R}^m} \int_{\mathbb{R}^n \setminus R(\bm{u}, \delta_1)} p(\bm{f_\mathrm{*}} \mid \bm{f}) p(\bm{y} \mid \bm{f}) p(\bm{f} \mid \bm{u}) p(\bm{u}) d\bm{f} d\bm{u} \geq 0.    \nonumber  
\end{align} 
Hence, the following holds: 
\begin{align}
|p(\bm{f_\mathrm{*}} \mid \bm{y}) - p(\bm{f_\mathrm{*}} \mid \bm{y}, \bm{\omega})| 
&\leq B_1(\bm{f_\mathrm{*}}) \epsilon(\delta_1, \delta_2) + B_1(\bm{f_\mathrm{*}}) (\exp(|\mathcal{L} - \mathcal{E}|) - 1) + B_2(\bm{f_\mathrm{*}}) \delta_1&   \nonumber  \\ 
&\leq B_1(\bm{f_\mathrm{*}}) \epsilon(\delta_1, \delta_2) + B_1(\bm{f_\mathrm{*}}) (a_1 \epsilon(\delta_1, \delta_2) + a_2 \delta_1) + B_2(\bm{f_\mathrm{*}}) \delta_1&   \nonumber  \\ 
&= (1 + a_1) B_1(\bm{f_\mathrm{*}}) \epsilon(\delta_1, \delta_2) + (a_2 B_1(\bm{f_\mathrm{*}}) + B_2(\bm{f_\mathrm{*}})) \delta_1&   \nonumber
\end{align} 
for any $\bm{f_\mathrm{*}} \in \mathbb{R}^{n_\mathrm{*}}$. 
Using this result, the following holds: 
\begin{align}
  \bigl\lVert\mathbb{E}_{p(\bm{f_\mathrm{*}} \mid \bm{y})} [\bm{f_\mathrm{*}}] - \mathbb{E}_{p(\bm{f_\mathrm{*}} \mid \bm{y}, \bm{\omega})} [\bm{f_\mathrm{*}}]\bigr\rVert
  &\leq \int_{\mathbb{R}^{n_\mathrm{*}}} \lVert\bm{f_\mathrm{*}}\rVert |p(\bm{f_\mathrm{*}} \mid \bm{y}) - p(\bm{f_\mathrm{*}} \mid \bm{y}, \bm{\omega})| d\bm{f_\mathrm{*}}&   \nonumber  \\   
  &\leq \int_{\mathbb{R}^{n_\mathrm{*}}} \lVert\bm{f_\mathrm{*}}\rVert ((1 + a_1) B_1(\bm{f_\mathrm{*}}) \epsilon(\delta_1, \delta_2) + (a_2 B_1(\bm{f_\mathrm{*}}) + B_2(\bm{f_\mathrm{*}})) \delta_1) d\bm{f_\mathrm{*}}&   \nonumber  \\ 
  &= (1 + a_1) b_1 \epsilon(\delta_1, \delta_2) + (a_2 b_1 + b_2) \delta_1.&    \nonumber  
\end{align} 
\end{proof}
The upper bounds in \cref{bound1} and \cref{bound2} contains $\delta_1$ and $\delta_2$, which can be chosen arbitrarily.  
Let $\delta_1$ and $\delta_2$ be 
\begin{align}
  \delta_1 = \sqrt{\xi \lambda_1 \kappa^{-1}} + \sqrt{\xi \lambda_2},~~\delta_2 = \sqrt{\xi \lambda_2},  \nonumber     
\end{align} 
where $\xi \in [\xi_0, \infty)$. 
Using \cref{monotonically}, we have 
\begin{align}
  \delta_1 + \epsilon(\delta_1, \delta_2)  
  &= (\sqrt{\xi \lambda_1 \kappa^{-1}} + \sqrt{\xi \lambda_2}) + \epsilon(\sqrt{\xi \lambda_1 \kappa^{-1}} + \sqrt{\xi \lambda_2}, \sqrt{\xi \lambda_2})&  \nonumber  \\  
  &= (\sqrt{\xi \lambda_1 \kappa^{-1}} + \sqrt{\xi \lambda_2}) + F(m, \xi m) + F(n, \xi n)&  \nonumber  \\  
  &= O(\sqrt{\xi \lambda_1 \kappa^{-1}} + \sqrt{\xi \lambda_2} + F(m, \xi m)).&  \nonumber  
\end{align} 
Therefore, the following holds: 
\begin{align}
  \mathcal{L} - \mathcal{E} = O(\eta), 
  ~~\bigl\lVert\mathbb{E}_{p(\bm{f_\mathrm{*}} \mid \bm{y})} [\bm{f_\mathrm{*}}] - \mathbb{E}_{p(\bm{f_\mathrm{*}} \mid \bm{y}, \bm{\omega})} [\bm{f_\mathrm{*}}]\bigr\rVert = O(\eta). \nonumber     
\end{align}

\subsection{Optimization of $\kappa$ and $\eta$ in \Cref{example_inputs}}  
\label{optimization_toy}  
We used the trust region method to find $\bm{u}$ in $\kappa$. 
To avoid the computation of the inverse of $\bm{K_\mathrm{uu}}$, we represent $\kappa$ as follows: 
\begin{align}
  \kappa = \inf_{\bm{u} \in \mathbb{R}^m; \lvert(\bm{W_\mathrm{fu}} \bm{K_\mathrm{uu}} - \bm{K_\mathrm{fu}}) \bm{u}\rvert = 1} \frac{n}{m} \lvert \bm{K_\mathrm{uu}} \bm{u}\rvert^2. \nonumber 
\end{align}
The initial value of $\bm{u}$ is $\lvert(\bm{W_\mathrm{fu}} \bm{K_\mathrm{uu}} - \bm{K_\mathrm{fu}}) [1, \cdots, 1]\rvert^{-1} [1, \cdots, 1]$. 
We employed the L-BFGS-B algorithm to optimize $\xi$ within $\eta$, initializing it to $\xi_0$. 
These processes were implemented with version 1.7.3 of the SciPy software.

\newpage 
\section{Sample Quasi-Likelihood} 
\subsection{Mariginal Likelihood and Posterior Distribution}
\label{Lap_likelihood_quasi}
The following matrix property is widely recognized. 
\begin{property}
  Let $\bm{A_\mathrm{uu}}$ and $\bm{B_\mathrm{uu}}$ denote $m \times m$ positive definite matrices, 
  The following holds: 
  \begin{align}
    (\bm{A_\mathrm{uu}} + \bm{B_\mathrm{uu}})^{-1} = \bm{A_\mathrm{uu}}^{-1} - \bm{A_\mathrm{uu}}^{-1} (\bm{A_\mathrm{uu}}^{-1} + \bm{B_\mathrm{uu}}^{-1})^{-1} \bm{A_\mathrm{uu}}^{-1}. \nonumber 
  \end{align}
\end{property}
Using this property, the following holds: 
\begin{align}
&\log \mathcal{N} (\bm{\bar{y}}; \bm{u}, \bm{V_\mathrm{uu}}) p(\bm{u})&    \nonumber   \\ 
&= - \frac{m}{2} \log (2 \pi) - \frac{1}{2} \log \lvert \bm{V_\mathrm{uu}}\rvert - \frac{1}{2} (\bm{u} - \bm{\bar{y}})^\top \bm{V_\mathrm{uu}}^{-1} (\bm{u} - \bm{\bar{y}})& \nonumber \\ 
&~~~~- \frac{m}{2} \log (2 \pi) - \frac{1}{2} \log \lvert \bm{K_\mathrm{uu}}\rvert - \frac{1}{2} \bm{u}^\top {\bm{K_\mathrm{uu}}}^{-1}\bm{u}&    \nonumber   \\ 
&= - \frac{m}{2} \log (2 \pi) - \frac{1}{2} \log \lvert \bm{V_\mathrm{uu}}\rvert - \frac{m}{2} \log (2 \pi) - \frac{1}{2} \log \lvert \bm{K_\mathrm{uu}}\rvert&  \nonumber   \\    
&~~~~- \frac{1}{2} \bm{u}^\top (\bm{V_\mathrm{uu}}^{-1} + {\bm{K_\mathrm{uu}}}^{-1}) \bm{u} 
     + \frac{1}{2} \bm{\bar{y}}^\top \bm{V_\mathrm{uu}}^{-1} \bm{u} 
     + \frac{1}{2} \bm{u}^\top \bm{V_\mathrm{uu}}^{-1} \bm{\bar{y}} 
     - \frac{1}{2} \bm{\bar{y}}^\top \bm{V_\mathrm{uu}}^{-1} \bm{\bar{y}} \nonumber   \\ 
&= - \frac{m}{2} \log (2 \pi) - \frac{1}{2} \log \lvert \bm{V_\mathrm{uu}}\rvert - \frac{1}{2} \log \lvert \bm{K_\mathrm{uu}}\rvert - \frac{1}{2} \log \lvert \bm{V_\mathrm{uu}}^{-1} + {\bm{K_\mathrm{uu}}}^{-1} \rvert&  \nonumber   \\  
&~~~~- \frac{1}{2} \bm{\bar{y}}^\top (\bm{V_\mathrm{uu}}^{-1} - \bm{V_\mathrm{uu}}^{-1} (\bm{V_\mathrm{uu}}^{-1} + \bm{K_\mathrm{uu}}^{-1})^{-1} \bm{V_\mathrm{uu}}^{-1}) \bm{\bar{y}}&  \nonumber   \\   
&~~~~- \frac{m}{2} \log (2 \pi) - \frac{1}{2} \log \lvert (\bm{V_\mathrm{uu}}^{-1} + {\bm{K_\mathrm{uu}}}^{-1})^{-1} \rvert&    \nonumber   \\ 
&~~~~- \frac{1}{2} (\bm{u} - (\bm{V_\mathrm{uu}}^{-1} + {\bm{K_\mathrm{uu}}}^{-1})^{-1} {\bm{V_\mathrm{uu}}^{-1}} \bm{\bar{y}})^\top (\bm{V_\mathrm{uu}}^{-1} + {\bm{K_\mathrm{uu}}}^{-1}) (\bm{u} - (\bm{V_\mathrm{uu}}^{-1} + {\bm{K_\mathrm{uu}}}^{-1})^{-1} {\bm{V_\mathrm{uu}}^{-1}} \bm{\bar{y}})& \nonumber   \\  
&= \mathcal{Q} + \log \mathcal{N} (\bm{u}; (\bm{V_\mathrm{uu}}^{-1} + {\bm{K_\mathrm{uu}}}^{-1})^{-1} {\bm{V_\mathrm{uu}}^{-1}} \bm{\bar{y}}, (\bm{V_\mathrm{uu}}^{-1} + {\bm{K_\mathrm{uu}}}^{-1})^{-1}).& \nonumber  
\end{align} 
Here the following property of Gaussian distribution is well known. 
\begin{property}
  \label{Wu}
  Let $\bm{u_0}$ denote $m$ real-valued vector, 
  $\bm{A_\mathrm{uu}}$ denote $m \times m$ positive definite matrix, 
  $\bm{B_\mathrm{*u}}$ denote $n_\mathrm{*} \times m$ real-valued matrix,  
  and $\bm{C_\mathrm{**}}$ denote $n_\mathrm{*} \times n_\mathrm{*}$ positive definite matrix. 
  Then we have 
  \begin{align}
    \int_{\mathbb{R}^m} \mathcal{N} (\bm{f_\mathrm{*}}; \bm{B_\mathrm{*u}} \bm{u}, \bm{C_\mathrm{**}}) \mathcal{N} (\bm{u}; \bm{u_0}, \bm{A_\mathrm{uu}}) d \bm{u}  
    = \mathcal{N} (\bm{f_\mathrm{*}}; \bm{B_\mathrm{*u}} \bm{u_0}, \bm{C_\mathrm{**}} + \bm{B_\mathrm{*u}} \bm{A_\mathrm{uu}} \bm{B_\mathrm{u*}}).   \nonumber       
  \end{align} 
\end{property}
Let the matrices in \cref{Wu} be 
\begin{align}
\bm{u_0} &= (\bm{V_\mathrm{uu}}^{-1} + {\bm{K_\mathrm{uu}}}^{-1})^{-1} \bm{V_\mathrm{uu}}^{-1} \bm{\bar{y}},&  \nonumber \\   
\bm{A_\mathrm{uu}} &= (\bm{V_\mathrm{uu}}^{-1} + {\bm{K_\mathrm{uu}}}^{-1})^{-1},& \nonumber \\ 
\bm{B_\mathrm{*u}} &= \bm{K_\mathrm{*u}} \bm{K_\mathrm{uu}}^{-1},& \nonumber \\ 
\bm{C_\mathrm{**}} &= \bm{K_\mathrm{**}} - \bm{K_\mathrm{*u}} \bm{K_\mathrm{uu}}^{-1} \bm{K_\mathrm{u*}}.& \nonumber 
\end{align} 
Then we have 
\begin{align}
  \bm{B_\mathrm{*u}} \bm{u_0} 
  &= \bm{K_\mathrm{*u}} \bm{K_\mathrm{uu}}^{-1} (\bm{V_\mathrm{uu}}^{-1} + {\bm{K_\mathrm{uu}}}^{-1})^{-1} \bm{V_\mathrm{uu}}^{-1} \bm{\bar{y}}&   \nonumber \\     
  &= \bm{K_\mathrm{*u}} \bm{K_\mathrm{uu}}^{-1} (\bm{V_\mathrm{uu}}^{-1} + \bm{K_\mathrm{uu}}^{-1})^{-1} ((\bm{V_\mathrm{uu}}^{-1} + \bm{K_\mathrm{uu}}^{-1}) - \bm{K_\mathrm{uu}}^{-1}) \bm{\bar{y}}&   \nonumber \\     
  &= \bm{K_\mathrm{*u}} (\bm{K_\mathrm{uu}}^{-1} - \bm{K_\mathrm{uu}}^{-1} (\bm{V_\mathrm{uu}}^{-1} + \bm{K_\mathrm{uu}}^{-1})^{-1} \bm{K_\mathrm{uu}}^{-1}) \bm{\bar{y}}&   \nonumber \\     
  &= \bm{K_\mathrm{*u}} (\bm{K_\mathrm{uu}} + \bm{V_\mathrm{uu}})^{-1} \bm{\bar{y}}&   \nonumber \\  
  &= \bm{\mu_q},&   \nonumber \\     
  \bm{C_\mathrm{**}} + \bm{B_\mathrm{*u}} \bm{A_\mathrm{uu}} \bm{B_\mathrm{u*}} 
  &= \bm{K_\mathrm{**}} - \bm{K_\mathrm{*u}} \bm{K_\mathrm{uu}}^{-1} \bm{K_\mathrm{u*}} + \bm{K_\mathrm{*u}} \bm{K_\mathrm{uu}}^{-1} (\bm{V_\mathrm{uu}}^{-1} + {\bm{K_\mathrm{uu}}}^{-1})^{-1} \bm{K_\mathrm{uu}}^{-1} \bm{K_\mathrm{u*}}&   \nonumber \\     
  &= \bm{K_\mathrm{**}} - \bm{K_\mathrm{*u}} (\bm{K_\mathrm{uu}}^{-1} - \bm{K_\mathrm{uu}}^{-1} (\bm{V_\mathrm{uu}}^{-1} + {\bm{K_\mathrm{uu}}}^{-1})^{-1} \bm{K_\mathrm{uu}}^{-1}) \bm{K_\mathrm{u*}}.&   \nonumber \\ 
  &= \bm{K_\mathrm{**}} - \bm{K_\mathrm{*u}} (\bm{K_\mathrm{uu}} + \bm{V_\mathrm{uu}})^{-1} \bm{K_\mathrm{u*}}&   \nonumber \\  
  &= \bm{\Sigma_q}.&   \nonumber
\end{align} 
Consequently, the following holds: 
\begin{align}
\int_{\mathbb{R}^m} q(\bm{f_\mathrm{*}} \mid \bm{u}) q(\bm{u} \mid \bm{\bar{y}}) d \bm{u} = \mathcal{N} (\bm{f_\mathrm{*}}; \bm{\mu_q}, \bm{\Sigma_q}),  \nonumber 
\end{align} 
where 
\begin{align}
  q(\bm{u} \mid \bm{\bar{y}}) \equiv \mathcal{N} (\bm{u}; (\bm{V_\mathrm{uu}}^{-1} + {\bm{K_\mathrm{uu}}}^{-1})^{-1} \bm{V_\mathrm{uu}}^{-1} \bm{\bar{y}}, (\bm{V_\mathrm{uu}}^{-1} + {\bm{K_\mathrm{uu}}}^{-1})^{-1}), \nonumber \\  
  q(\bm{f_\mathrm{*}} \mid \bm{u}) \equiv \mathcal{N} (\bm{f_\mathrm{*}}; \bm{K_\mathrm{*u}} \bm{K_\mathrm{uu}}^{-1} \bm{u}, \bm{K_\mathrm{**}} - \bm{K_\mathrm{*u}} \bm{K_\mathrm{uu}}^{-1} \bm{K_\mathrm{u*}}).   \nonumber 
\end{align}

\subsection{Proof of \Cref{Lap_theorem}}
\label{Lap}
Considering the second order Taler series of $\log p(\bm{y} \mid \bm{W_\mathrm{fu}} \bm{u})$ around $\bm{\bar{y}}$, the following holds: 
\begin{align}
&\log p(\bm{y} \mid \bm{W_\mathrm{fu}} \bm{u}) p(\bm{u})&    \nonumber   \\ 
&= \log p(\bm{y} \mid \bm{W_\mathrm{fu}} \bm{u}) |_{\bm{u} = \bm{\bar{y}}} - \frac{1}{2} (\bm{u} - \bm{\bar{y}})^\top \bm{V_\mathrm{uu}}^{-1} (\bm{u} - \bm{\bar{y}}) + o_p (m) 
  - \frac{m}{2} \log (2 \pi) - \frac{1}{2} \log \lvert \bm{K_\mathrm{uu}}\rvert - \frac{1}{2} \bm{u}^\top {\bm{K_\mathrm{uu}}}^{-1}\bm{u}&    \nonumber   \\ 
&= \log p(\bm{y} \mid \bm{W_\mathrm{fu}} \bm{u}) |_{\bm{u} = \bm{\bar{y}}} + \frac{m}{2} \log (2 \pi) + \frac{1}{2} \log \lvert \bm{V_\mathrm{uu}}\rvert&    \nonumber   \\  
&~~~~+ \mathcal{Q} + \log \mathcal{N} (\bm{u}; (\bm{V_\mathrm{uu}}^{-1} + {\bm{K_\mathrm{uu}}}^{-1})^{-1} {\bm{V_\mathrm{uu}}^{-1}} \bm{\bar{y}}, (\bm{V_\mathrm{uu}}^{-1} + {\bm{K_\mathrm{uu}}}^{-1})^{-1})  + o_p (m).&    \nonumber  
\end{align} 
Therefore, the following holds: 
\begin{align}
  \mathcal{E} - \mathcal{Q} = \log p(\bm{y} \mid \bm{W_\mathrm{fu}} \bm{u}) |_{\bm{u} = \bm{\bar{y}}} + \frac{m}{2} \log (2 \pi) + \frac{1}{2} \log \lvert \bm{V_\mathrm{uu}}\rvert + o_p (m).  \nonumber 
\end{align} 
Additionally, the following holds: 
\begin{align}
  \exp(-\mathcal{E}) p(\bm{y} \mid \bm{W_\mathrm{fu}} \bm{u}) p(\bm{u}) \xrightarrow[]{\mathrm{d}} \mathcal{N} (\bm{u}; (\bm{V_\mathrm{uu}}^{-1} + {\bm{K_\mathrm{uu}}}^{-1})^{-1} \bm{V_\mathrm{uu}}^{-1} \bm{\bar{y}}, (\bm{V_\mathrm{uu}}^{-1} + {\bm{K_\mathrm{uu}}}^{-1})^{-1}). \nonumber 
\end{align} 
Let the matrices in \cref{Wu} be 
\begin{align}
\bm{u_0} &= (\bm{V_\mathrm{uu}}^{-1} + {\bm{K_\mathrm{uu}}}^{-1})^{-1} \bm{V_\mathrm{uu}}^{-1} \bm{\bar{y}},& \nonumber \\    
\bm{A_\mathrm{uu}} &= (\bm{V_\mathrm{uu}}^{-1} + {\bm{K_\mathrm{uu}}}^{-1})^{-1},& \nonumber \\ 
\bm{B_\mathrm{*u}} &= \bm{K_\mathrm{*f}} \bm{K_\mathrm{ff}}^{-1} \bm{W_\mathrm{fu}},& \nonumber \\ 
\bm{C_\mathrm{**}} &= \bm{K_\mathrm{**}} - \bm{K_\mathrm{*f}} \bm{K_\mathrm{ff}}^{-1} \bm{K_\mathrm{f*}}.& \nonumber 
\end{align}
Then we have 
\begin{align}
  \bm{B_\mathrm{*u}} \bm{u_0} 
  &= \bm{K_\mathrm{*f}} \bm{K_\mathrm{ff}}^{-1} \bm{W_\mathrm{fu}} (\bm{V_\mathrm{uu}}^{-1} + {\bm{K_\mathrm{uu}}}^{-1})^{-1} \bm{V_\mathrm{uu}}^{-1} \bm{\bar{y}} = \bm{\mu_p},&   \nonumber \\     
  \bm{C_\mathrm{**}} + \bm{B_\mathrm{*u}} \bm{A_\mathrm{uu}} \bm{B_\mathrm{u*}}  
  &= \bm{K_\mathrm{**}} - \bm{K_\mathrm{*f}} \bm{K_\mathrm{ff}}^{-1} \bm{K_\mathrm{f*}} + \bm{K_\mathrm{*f}} \bm{K_\mathrm{ff}}^{-1} \bm{W_\mathrm{fu}} (\bm{V_\mathrm{uu}}^{-1} + {\bm{K_\mathrm{uu}}}^{-1})^{-1} \bm{W_\mathrm{uf}} \bm{K_\mathrm{ff}}^{-1} \bm{K_\mathrm{f*}} = \bm{\Sigma_p}.&   \nonumber 
\end{align} 
Consequently, the following holds: 
\begin{align} 
  p(\bm{f_\mathrm{*}} \mid \bm{y}, \bm{\omega}) \xrightarrow[]{\mathrm{d}} \mathcal{N} (\bm{f_\mathrm{*}}; \bm{\mu_p}, \bm{\Sigma_p}).  \nonumber 
\end{align}

\subsection{Proof of \Cref{approx_matrix}}  
\label{proof_approx_matrix}
The following holds: 
\begin{align} 
  \bm{K_\mathrm{ff}} \bm{K_\mathrm{ff}}^{-1} \bm{W_\mathrm{fu}} = \bm{W_\mathrm{fu}} 
  &\Rightarrow (\bm{W_\mathrm{fu}} \bm{K_\mathrm{uu}} \bm{W_\mathrm{uf}} + O(\beta)) \bm{K_\mathrm{ff}}^{-1} \bm{W_\mathrm{fu}} = \bm{W_\mathrm{fu}}& \nonumber  \\    
  &\Rightarrow (\bm{K_\mathrm{uu}} \bm{W_\mathrm{uf}} + O(\beta)) \bm{K_\mathrm{ff}}^{-1} \bm{W_\mathrm{fu}} = \diag(1, \cdots, 1)& \nonumber  \\   
  &\Rightarrow (\bm{W_\mathrm{uf}} + O(\beta)) \bm{K_\mathrm{ff}}^{-1} \bm{W_\mathrm{fu}} = \bm{K_\mathrm{uu}}^{-1}& \nonumber  \\  
  &\Rightarrow \bm{W_\mathrm{uf}} \bm{K_\mathrm{ff}}^{-1} \bm{W_\mathrm{fu}} + O(\beta) \bm{K_\mathrm{ff}}^{-1} \bm{W_\mathrm{fu}} = \bm{K_\mathrm{uu}}^{-1}& \nonumber  \\  
  &\Rightarrow \bm{W_\mathrm{uf}} \bm{K_\mathrm{ff}}^{-1} \bm{W_\mathrm{fu}} = \bm{K_\mathrm{uu}}^{-1} + O(\beta).& \nonumber 
\end{align} 
Consequently, the following holds: 
\begin{align} 
  \bm{\mu_p} 
  &= \bm{K_\mathrm{*f}} \bm{K_\mathrm{ff}}^{-1} \bm{W_\mathrm{fu}} (\bm{V_\mathrm{uu}}^{-1} + {\bm{K_\mathrm{uu}}}^{-1})^{-1} \bm{V_\mathrm{uu}}^{-1} \bm{\bar{y}}& \nonumber  \\ 
  &= (\bm{K_\mathrm{*u}} \bm{W_\mathrm{uf}} + O(\beta)) \bm{K_\mathrm{ff}}^{-1} \bm{W_\mathrm{fu}} (\bm{V_\mathrm{uu}}^{-1} + {\bm{K_\mathrm{uu}}}^{-1})^{-1} \bm{V_\mathrm{uu}}^{-1} \bm{\bar{y}}& \nonumber  \\ 
  &= (\bm{K_\mathrm{*u}} \bm{W_\mathrm{uf}} \bm{K_\mathrm{ff}}^{-1} \bm{W_\mathrm{fu}} + O(\beta) \bm{K_\mathrm{ff}}^{-1} \bm{W_\mathrm{fu}}) (\bm{V_\mathrm{uu}}^{-1} + {\bm{K_\mathrm{uu}}}^{-1})^{-1} \bm{V_\mathrm{uu}}^{-1} \bm{\bar{y}}& \nonumber  \\ 
  &= (\bm{K_\mathrm{*u}} \bm{K_\mathrm{uu}}^{-1} + O(\beta)) (\bm{V_\mathrm{uu}}^{-1} + {\bm{K_\mathrm{uu}}}^{-1})^{-1} \bm{V_\mathrm{uu}}^{-1} \bm{\bar{y}}& \nonumber  \\ 
  &= \bm{K_\mathrm{*u}} \bm{K_\mathrm{uu}}^{-1} (\bm{V_\mathrm{uu}}^{-1} + {\bm{K_\mathrm{uu}}}^{-1})^{-1} \bm{V_\mathrm{uu}}^{-1} \bm{\bar{y}} + O(\beta)& \nonumber  \\ 
  &= \bm{\mu_q} + O(\beta).& \nonumber 
\end{align} 
Similarly, the following holds: 
\begin{align} 
  \bm{\Sigma_p} 
  &= \bm{K_\mathrm{**}} - \bm{K_\mathrm{*f}} \bm{K_\mathrm{ff}}^{-1} \bm{K_\mathrm{f*}} + \bm{K_\mathrm{*f}} \bm{K_\mathrm{ff}}^{-1} \bm{W_\mathrm{fu}} (\bm{V_\mathrm{uu}}^{-1} + {\bm{K_\mathrm{uu}}}^{-1})^{-1} \bm{W_\mathrm{uf}} \bm{K_\mathrm{ff}}^{-1} \bm{K_\mathrm{f*}}& \nonumber  \\  
  &= \bm{K_\mathrm{**}} - (\bm{K_\mathrm{*u}} \bm{W_\mathrm{uf}} + O(\beta)) \bm{K_\mathrm{ff}}^{-1} (\bm{W_\mathrm{fu}} \bm{K_\mathrm{u*}} + O(\beta))& \nonumber  \\  
  &~~~~+ (\bm{K_\mathrm{*u}} \bm{W_\mathrm{uf}} + O(\beta)) \bm{K_\mathrm{ff}}^{-1} \bm{W_\mathrm{fu}} (\bm{V_\mathrm{uu}}^{-1} + {\bm{K_\mathrm{uu}}}^{-1})^{-1} \bm{W_\mathrm{uf}} \bm{K_\mathrm{ff}}^{-1} (\bm{W_\mathrm{fu}} \bm{K_\mathrm{u*}} + O(\beta))& \nonumber  \\  
  &= \bm{K_\mathrm{**}} - (\bm{K_\mathrm{*u}} \bm{W_\mathrm{uf}} \bm{K_\mathrm{ff}}^{-1} \bm{W_\mathrm{fu}} \bm{K_\mathrm{u*}} + O(\beta) \bm{K_\mathrm{ff}}^{-1} \bm{W_\mathrm{fu}} \bm{K_\mathrm{u*}} + \bm{K_\mathrm{*u}} \bm{W_\mathrm{uf}} \bm{K_\mathrm{ff}}^{-1} O(\beta) + O(\beta^2))& \nonumber  \\  
  &~~~~+ (\bm{K_\mathrm{*u}} \bm{W_\mathrm{uf}} \bm{K_\mathrm{ff}}^{-1} \bm{W_\mathrm{fu}} + O(\beta) \bm{K_\mathrm{ff}}^{-1} \bm{W_\mathrm{fu}}) (\bm{V_\mathrm{uu}}^{-1} + {\bm{K_\mathrm{uu}}}^{-1})^{-1} (\bm{W_\mathrm{uf}} \bm{K_\mathrm{ff}}^{-1} \bm{W_\mathrm{fu}} \bm{K_\mathrm{u*}} + \bm{W_\mathrm{uf}} \bm{K_\mathrm{ff}}^{-1} O(\beta))& \nonumber  \\  
  &= \bm{K_\mathrm{**}} - (\bm{K_\mathrm{*u}} \bm{K_\mathrm{uu}}^{-1} \bm{K_\mathrm{u*}} + O(\beta))& \nonumber  \\    
  &~~~~+ (\bm{K_\mathrm{*u}} \bm{K_\mathrm{uu}}^{-1} + O(\beta)) (\bm{V_\mathrm{uu}}^{-1} + {\bm{K_\mathrm{uu}}}^{-1})^{-1} (\bm{K_\mathrm{uu}}^{-1} \bm{K_\mathrm{u*}} + O(\beta))& \nonumber  \\  
  &= \bm{K_\mathrm{**}} - \bm{K_\mathrm{*u}} \bm{K_\mathrm{uu}}^{-1} \bm{K_\mathrm{u*}}  
   + \bm{K_\mathrm{*u}} \bm{K_\mathrm{uu}}^{-1} (\bm{V_\mathrm{uu}}^{-1} + {\bm{K_\mathrm{uu}}}^{-1})^{-1} \bm{K_\mathrm{uu}}^{-1} \bm{K_\mathrm{u*}} + O(\beta + m \beta^2)& \nonumber  \\  
  &= \bm{\Sigma_q} + O(\beta + m \beta^2). \nonumber  
\end{align}

\subsection{Comparison with the Quasi-Likelihood}
\label{QL_Description}
By specifying only the relationship between the mean and the variance, a quasi-likelihood function exhibits certain properties within the framework of generalized linear models. 
This is advantageous for parameter estimation when the distribution of observations is complex. 
Here we consider applying quasi-likelihood to Gaussian process regression. 
Let the quasi-likelihood function $Q : \mathcal{Y}^n \times \mathbb{R}^n \rightarrow \mathbb{R}$ satisfy    
\begin{align} 
  \frac{\partial Q(\bm{y}, \bm{f})}{\partial \bm{f}} = \bm{V_\mathrm{ff}}^{-1} (\bm{y} - \bm{f}), \nonumber 
\end{align}   
where $\bm{V_\mathrm{ff}} \equiv \diag(v(f(\bm{x}_1)), \cdots, v(f(\bm{x}_n)))$. 
For $p(\bm{y} \mid \bm{f}) = \mathcal{N} (\bm{y}; \bm{f}, \bm{V_\mathrm{ff}})$, learning and inference are challenging because $\bm{V_\mathrm{ff}}$ depends on $\bm{f}$.  
The difference from the sample quasi-likelihood function lies in $\bm{V_\mathrm{ff}}$ and $\bm{V_\mathrm{uu}}$. 
Although $\bm{V_\mathrm{ff}}$ contains parameters, $\bm{V_\mathrm{uu}}$ depends on the summary statistics.

\newpage 
\section{Spatial Modeling}

\subsection{Likelihood Function} 
\label{example_likelihood} 

\begin{table*}[h]
    \centering
    \begin{threeparttable}[]
    \begin{tabular}{cccc} 
        \toprule
        Function & $p(y \:|\: f(\boldsymbol{x}))$ & $\frac{\partial \log p(y \:|\: f(\boldsymbol{x}))}{\partial f(\boldsymbol{x})}$ & $\frac{\partial^2 \log p(y \:|\: f(\boldsymbol{x}))}{\partial f(\boldsymbol{x})^2}$    \\ 	
        \midrule 
        Gaussian 
        &  $\frac{1}{\sqrt {2 \pi} \sigma} \exp \bigl(- \frac{(y - f(\boldsymbol{x}))^2}{2 \sigma^2}\bigr)$  
        & $\frac{y - f(\boldsymbol{x})}{\sigma^2}$  
        & $- \frac{1}{\sigma^2}$  \\ 
        Poisson 
        &  $\frac{\exp(- \exp(f(\boldsymbol{x})) + y f(\boldsymbol{x}))}{y !}$  
        & $- \exp(f(\boldsymbol{x})) + y$  
        & $- \exp(f(\boldsymbol{x}))$   \\ 
        Bernoulli 
        &  $\Phi_f^y (1 - \Phi_f)^{1 - y}$   
        & $\frac{\phi_f}{\Phi_f} y + \frac{- \phi_f}{1 - \Phi_f} (1 - y)$ 
        &  $- \frac{f \phi_f \Phi_f + \phi_f^2}{\Phi_f^2} y - \frac{\phi_f^2 - f \phi_f + f \phi_f \Phi_f}{(1 - \Phi_f)^2} (1 - y)$  \\ 
        \bottomrule 
    \end{tabular}
    \end{threeparttable}
    \caption{
    Examples of likelihood functions. 
    Let $\Phi_f \equiv \Phi(f(\boldsymbol{x}))$ and $\phi_f \equiv \frac{\partial \Phi(f(\boldsymbol{x}))}{\partial f(\boldsymbol{x})}$, 
    where $\Phi$ is the cumulative distribution function of standard normal distribution.   
    }
    \label{types}
\end{table*}

\subsection{Prediction Performance} 
\label{prediction_performance} 

\begin{table*}[h]
  \centering
  \begin{threeparttable}[]
  \begin{tabular}{ccccccccc} 
  \toprule
  \multirow{2.5}{*}{Output} & \multirow{2.5}{*}{Likelihood} & \multicolumn{2}{c}{Complete VI} & \multicolumn{2}{c}{Summarized VI} & \multicolumn{2}{c}{Our Approach} \\ 	
  \cmidrule(lr){3-4} \cmidrule(lr){5-6} \cmidrule(lr){7-8} &  & RMSE & Time & RMSE & Time & RMSE & Time \\ 
  \midrule 
  MedInc & Gaussian & $0.943 \pm 0.019$ & $3 \pm 1$ & $0.991 \pm 0.017$ & $3 \pm 1$ & $0.987 \pm 0.008$ & $\bm{0 \pm 0}$ \\ 
  MedInc & Poisson & $0.953 \pm 0.020$ & $18 \pm 7$ & $0.996 \pm 0.015$ & $15 \pm 7$ & $0.988 \pm 0.002$ & $\bm{0 \pm 0}$ \\ 
  \midrule 
  HouseAge & Gaussian & $0.947 \pm 0.041$ & $4 \pm 3$ & $\bm{0.990 \pm 0.009}$ & $4 \pm 2$ & $1.036 \pm 0.036$ & $\bm{0 \pm 0}$ \\ 
  HouseAge & Poisson & $0.971 \pm 0.048$ & $7 \pm 6$ & $\bm{0.998 \pm 0.008}$ & $8 \pm 7$ & $1.066 \pm 0.015$ & $\bm{0 \pm 0}$ \\ 
  \midrule 
  AveRooms & Gaussian & $0.991 \pm 0.016$ & $1 \pm 0$ & $1.000 \pm 0.004$ & $1 \pm 1$ & $\bm{0.990 \pm 0.027}$ & $\bm{0 \pm 0}$ \\ 
  AveRooms & Poisson & $0.994 \pm 0.015$ & $5 \pm 5$ & $0.999 \pm 0.006$ & $7 \pm 7$ & $\bm{0.969 \pm 0.005}$ & $\bm{0 \pm 0}$ \\ 
  \midrule 
  AveBedrms & Gaussian & $0.982 \pm 0.028$ & $3 \pm 3$ & $0.994 \pm 0.009$ & $4 \pm 3$ & $\bm{0.957 \pm 0.012}$ & $\bm{1 \pm 0}$ \\ 
  AveBedrms & Poisson & $0.966 \pm 0.028$ & $13 \pm 10$ & $0.987 \pm 0.011$ & $14 \pm 10$ & $\bm{0.966 \pm 0.012}$ & $\bm{0 \pm 0}$ \\ 
  \midrule 
  Population & Gaussian & $1.188 \pm 0.219$ & $1 \pm 1$ & $1.071 \pm 0.130$ & $4 \pm 3$ & $\bm{1.000 \pm 0.003}$ & $\bm{0 \pm 0}$ \\ 
  Population & Poisson & $0.994 \pm 0.006$ & $14 \pm 10$ & $0.998 \pm 0.003$ & $13 \pm 10$ & $1.000 \pm 0.012$ & $\bm{0 \pm 0}$ \\ 
  \midrule 
  AveOccup & Gaussian & $1.001 \pm 0.004$ & $1 \pm 1$ & $1.001 \pm 0.004$ & $1 \pm 1$ & $1.001 \pm 0.004$ & $\bm{0 \pm 0}$ \\ 
  AveOccup & Poisson & $1.001 \pm 0.004$ & $8 \pm 7$ & $1.001 \pm 0.002$ & $7 \pm 7$ & $1.001 \pm 0.004$ & $\bm{0 \pm 0}$ \\ 
  \midrule 
  MedValue & Gaussian & $0.810 \pm 0.013$ & $4 \pm 1$ & $1.008 \pm 0.030$ & $4 \pm 1$ & $\bm{0.964 \pm 0.015}$ & $\bm{0 \pm 0}$ \\ 
  MedValue & Poisson & $0.806 \pm 0.014$ & $16 \pm 4$ & $1.055 \pm 0.057$ & $16 \pm 5$ & $\bm{0.985 \pm 0.018}$ & $\bm{0 \pm 0}$ \\ 
  \bottomrule 
  \end{tabular}
  \end{threeparttable} 
  \caption{
  Prediction performance of our approach with $0.8 \times 0.8$ grid size.  
  The displays are the same as \cref{performance2}. 
  }
  \label{pp8}
\end{table*}

\end{document}